\documentclass[letterpaper]{article} 
\usepackage{aaai24}  
\usepackage{times}  
\usepackage{helvet}  
\usepackage{courier}  
\usepackage[hyphens]{url}  
\usepackage{graphicx} 
\usepackage{natbib}  
\usepackage{caption} 
\usepackage{algorithm}
\usepackage{algorithmic}

\usepackage{newfloat}
\usepackage{listings}
\DeclareCaptionStyle{ruled}{labelfont=normalfont,labelsep=colon,strut=off} 
\floatstyle{ruled}
\newfloat{listing}{tb}{lst}{}
\floatname{listing}{Listing}
\usepackage{amsmath}
\usepackage{amssymb}
\usepackage{amsfonts}
\usepackage{amsthm}
\usepackage{multirow}
\usepackage{booktabs}
\usepackage{tabularx}

\newcommand{\calD}{{\cal D}}

\newcommand{\calF}{{\cal F}}

\newcommand{\calI}{{\cal I}}

\newcommand{\calW}{{\cal W}}
\newcommand{\calX}{{\cal X}}
\newcommand{\calY}{{\cal Y}}

\newcommand{\defeq}{\stackrel{\text { def. }}{=}}
\newcommand{\eps}{\varepsilon}

\newtheorem{theorem}{Theorem}
\newtheorem{corollary}[theorem]{Corollary}
\newtheorem{lemma}[theorem]{Lemma}

\newtheorem{assumption}{Assumption}

\newtheorem{lemmaS}{Lemma}

\usepackage{xcolor}
\newcommand{\zx}[1]{\textcolor{blue}{\textbf{Zikai:} #1}}
\newcommand{\nd}[1]{\textcolor{cyan}{\textbf{Nic:} #1}}

\title{FairWASP: Fast and Optimal Fair Wasserstein Pre-processing}
\author {
    Zikai Xiong\textsuperscript{\rm 1},
    Niccol\`o Dalmasso\textsuperscript{\rm 2,}\footnote{Corresponding Author},
    Alan Mishler\textsuperscript{\rm 2}, \\
    Vamsi K. Potluru\textsuperscript{\rm 2},
    Tucker Balch\textsuperscript{\rm 2},
    Manuela Veloso\textsuperscript{\rm 2}
}
\affiliations {
    \textsuperscript{\rm 1}Massachusetts Institute of Technology\\
    \textsuperscript{\rm 2}J.P. Morgan AI Research, New York\\
    {zikai@mit.edu, niccolo.dalmasso@jpmchase.com}
}

\usepackage{bibentry}

\begin{document}

\maketitle

\begin{abstract}

Recent years have seen a surge of machine learning approaches aimed at reducing disparities in model outputs across different subgroups. In many settings, training data may be used in multiple downstream applications by different users, which means it may be most effective to intervene on the training data itself. In this work, we present FairWASP, a novel pre-processing approach designed to reduce disparities in classification datasets without modifying the original data. FairWASP returns sample-level weights such that the reweighted dataset minimizes the Wasserstein distance to the original dataset while satisfying (an empirical version of) demographic parity, a popular fairness criterion. We show theoretically that integer weights are optimal, which means our method can be equivalently understood as duplicating or eliminating samples. FairWASP can therefore be used to construct datasets which can be fed into any classification method, not just methods which accept sample weights.
Our work is based on reformulating the pre-processing task as a large-scale mixed-integer program (MIP), for which we propose a highly efficient algorithm based on the cutting plane method. Experiments demonstrate that our proposed optimization algorithm significantly outperforms state-of-the-art commercial solvers in solving both the MIP and its linear program relaxation. Further experiments highlight the competitive performance of FairWASP in reducing disparities while preserving accuracy in downstream classification settings.

\end{abstract}

\section{Introduction}



Machine learning is increasingly involved in decision making that impacts people's lives \cite{sloane2022SiliconValleyLove, zhang2022ShiftingMachineLearning}. There is concern that models may inherit from the data bias against subgroups defined by race, gender, or other protected characteristics. Accordingly, there is a vast literature on methods to make machine learning models ``fair.’’ While there is no consensus about what it means for a model to be fair or unfair in a given setting, these methods commonly aim to minimize disparities in model outputs or model performance across different subgroups.

Fair machine learning methods are traditionally divided into three categories: (i) \emph{pre-processing} methods intervene on the training data, (ii) \emph{in-processing} methods apply constraints or regularizers during the model training process itself, and (iii) \emph{post-processing} methods alter the outputs of previously trained models. See \citet{hort2022BiasMitigationMachinea} for a recent review of methods across all three categories.

Among these three, no one category of methods clearly dominates the others in terms of performance.
Pre-processing methods are useful when the person who generates or maintains a dataset is not the same as the person who will be using it to train a model \cite{feldman2015certifying}, or when a dataset may be used to train multiple models. These methods typically require no knowledge of downstream models, so they are in principle compatible with any subsequent machine learning procedure.

Many pre-processing methods operate by changing the feature values or labels of the training data \cite{calders2009BuildingClassifiersIndependency, zliobaite2011HandlingConditionalDiscrimination}, subsampling or oversampling the data \cite{kamiran2010ClassificationNoDiscrimination, yan2020FairClassBalancing, chakraborty2021BiasMachineLearning, salazar2021FAWOSFairnessAwareOversampling}, and/or generating synthetic data \cite{xu2018FairGANFairnessawareGenerativea, salimi2019InterventionalFairnessCausal}. In high-stakes settings such as finance and healthcare, however, it may be unethical or even illegal to alter customer or patient attributes or labels, e.g. with current data regulations in the European Union (GDPR) or health information (HIPAA) in the United States.
Maintaining separate, modified versions of datasets is possible but may be costly for large datasets. An alternative is to learn a set of sample weights that can be passed to a learning method at training time \cite{calders2009BuildingClassifiersIndependency, chai2022FairnessAdaptiveWeights, jiang2020IdentifyingCorrectingLabel, li2022AchievingFairnessNo}. While there are many methods that do this, they focus on satisfying fairness constraints without providing guarantees about how much they alter the distribution of the data.

\subsubsection{Contributions} We present FairWASP, a novel pre-processing method that learns a set of sample weights for classification datasets without modifying the training data. FairWASP minimizes the Wasserstein distance between the original and reweighted datasets while ensuring that the reweighted dataset satisfies (an empirical version of) demographic parity, a popular fairness criterion, which we detail in Section~\ref{sec:FairWASP}. Our contributions are as follows:

\begin{enumerate}
    \item Since directly solving the target optimization problem is computationally infeasible, we provide a three-step reformulation that leads to a tractable linear program in Section \ref{sec: reformulation}. We prove that the solution to this linear program is a solution to the original problem under a mild assumption and show theoretically that, over the set of real-valued weights, integer-valued weights are in fact optimal. This means that FairWASP can be understood equivalently as indicating which samples (rows) of the dataset should be duplicated or deleted at training time, so it is compatible with any downstream classification algorithm, not just algorithms that accept sample weights.
    \item We contribute a highly efficient algorithm to solve the reformulated linear program (Section~\ref{sec:cuttin_plane}), that vastly outperforms state-of-the-art commercial solvers.
    \item We extend FairWASP to satisfy a separate but equivalent definition of demographic parity (Section~\ref{sec: fairwasp-pw}) by leveraging the linear program reformulation above.
    \item We empirically show that FairWASP achieves competitive performance in reducing disparities while preserving accuracy in downstream classification settings when compared to existing pre-processing methods (Section~\ref{sec: experiments}).
\end{enumerate}

See the Supplementary Materials for complete proofs of theoretical claims, more discussion, and details on our algorithm and experiments results.

\section{Background}\label{sec: background}

\subsubsection{Setup }
Consider a dataset of $n$ i.i.d. samples $\left\{Z_i = \left(D_i, X_i, Y_i\right)\right\}_{i=1}^n$ drawn from a joint distribution $p_{Z} = p_{D, X, Y}$ with domain $\mathcal{Z} = \mathcal{D} \times \mathcal{X} \times \mathcal{Y}$. In this context, $D$ represents one or more protected variables such as gender or race, $X$ indicates additional features used for decision-making, and $Y$ is the decision outcome. For example, $Y_i$ could represent a loan approval decision for individual $i$, based on demographic data $D_i$ and credit score $X_i$. Learning tasks typically aim at learning the conditional distribution $P(Y|X)$ or $P(Y|X,D)$ from the samples $\left\{Z_i\right\}_{i=1}^n$.
In this paper, we assume that the number of demographic classes $|\calD|$ and the number of outcome levels $|\calY|$ are significantly smaller than $n$. 


\subsubsection{Demographic Parity (DP)}
Demographic parity (DP), also known as statistical parity, requires an outcome variable to be statistically independent of a sensitive feature \cite{dwork2012FairnessAwareness}. This could mean, for example, that an algorithm used to screen resumes for interviews is required to recommend equal proportions of female and male applicants. DP is arguably the most widely studied fairness criterion to date \cite{hort2022BiasMitigationMachinea}. Violations of DP may be measured in different ways. For FairWASP, we adopt a measure similar to \citet{dwork2012FairnessAwareness} and \citet{calmon2017optimized}, namely the distances between the marginal distribution of an outcome variable and the distributions of that outcome variable conditional on levels of a sensitive feature. Additionally, we show in Section~\ref{sec: fairwasp-pw} that measuring DP as the distance between outcome distributions for each level of the sensitive feature \cite{calmon2017optimized} can also be reformulated in a similar way as the FairWASP optimization problem.

\subsubsection{Pre-processing via Reweighting}
\citet{calders2009BuildingClassifiersIndependency} proposed utilizing a set of sample weights based on the sensitive feature and the outcome variable to target DP. Since then, a variety of papers have utilized similar reweighting approaches \cite{kamiran2010ClassificationNoDiscrimination, jiang2020IdentifyingCorrectingLabel, chai2022FairnessAdaptiveWeights,  li2022AchievingFairnessNo}. However, previous papers provide no guarantees about how the sample weights will change the overall distribution of the data. If the weights alter the distribution of the data significantly, the downstream model might not learn the correct conditional distribution between target variables and features, i.e., $P(Y|X)$ or $P(Y|X, D)$. While minimizing data perturbation has been considered in pre-processing papers which seek to learn transformations of the data itself \cite{zemel2013learning, calmon2017optimized}, to our knowledge, FairWASP is the first reweighting approach that seeks to minimize the overall distributional distance from the original data.

\subsubsection{Wasserstein Distance} 
The general Wasserstein distance (or optimal transport metric) between two probability distributions $(\mu, \nu)$ supported on a metric space $\mathcal{X}$ is defined as the optimal objective of the (possibly infinite-dimensional) linear program (LP):
\begin{equation}\label{pro original wasserstrein distance}
	\mathcal{W}_c(\mu, \nu) \stackrel{\text{def.}}{=} \min_{\pi \in \Pi(\mu, \nu)} \int_{\mathcal{X} \times \mathcal{X}} c(x, y) \mathrm{d} \pi(x, y),
\end{equation}
where $\Pi(\mu, \nu)$ is the set of couplings composed of  joint probability distributions over the product space $\mathcal{X} \times \mathcal{X}$ with marginals $(\mu, \nu)$.
Equation \eqref{pro original wasserstrein distance} is also called the Kantorovitch formulation of optimal transport \cite{kantorovitch1958translocation}. Here, $c(x, y)$ represents the ``cost'' to move a unit of mass from $x$ to $y$. A typical choice in space $\calX$ with metric $d_{\calX}$ is $c(x, y)=d_{\mathcal{X}}(x, y)^p$ for $p \ge 1$, and then $\mathcal{W}_c^{1 / p}$ is referred to as the $p$-Wasserstein distance between probability measures. 
%
%
%
Using the Wasserstein distance between distributions is particularly useful as it provides a bound for functions applied to samples from those distributions. In other words, define the following deviation:
$$
d(\mu,  \nu)\defeq \sup_{f\in\calF} \left|\mathbb{E}_{z\sim \mu} f(z)- \mathbb{E}_{z\sim \nu} f(z) \right| \ ,
$$
where $\calF$ is a family of functions $f$. If $\calF = \textit{Lip}_1$, the class of Lipschitz-continuous functions with Lipschitz constant of $1$, then the deviation $d(\mu, \nu)$ is equal to the 1-Wasserstein distance \cite{santambrogio2015optimal,villani2009optimal}. Analogous bounds can be derived for the 2-Wasserstein distance  when $\calF = \left\{f \mid\|f\|_{\mathcal{S}^1(\mu)} \leq 1\right\}$, the class of functions with unitary norm over the Sobolev space $\mathcal{S} = \left\{f \in L^2 \mid \partial_{x_i} f \in L^2\right\}$ \cite{claici2018wasserstein}. This fact provides a theoretical intuition for downstream utility preservation, i.e., the closer two distributions are in Wasserstein distance, the more similar the downstream performance of learning models trained on such distributions is expected to be. Finally, the Wasserstein distance has been used to express fairness constraints in several in-processing methods \cite{chzhen2020FairRegressionWasserstein, chzhen2022minimax}. To our knowledge, however, it has not previously been used in a pre-processing setting.

\section{FairWASP Optimization Problem}\label{sec:FairWASP}

In this section we propose FairWASP, which casts dataset pre-processing as an optimization problem that aims at minimizing the distance to the original data distribution while satisfying fairness constraints.


Given a dataset $Z = \left\{\left(D_i, X_i, Y_i\right)\right\}_{i=1}^n$, we can write the reweighted distribution of the dataset as:
$$
p_{Z;\theta} \defeq \frac{1}{n} \sum_{i=1}^n \theta_i \delta_{Z_i},
$$
with $\{\theta_i\}_{i\in[n]}$ such that $\sum_i \theta_i = n$, and Dirac measures $\delta_{Z_i}$ centered on $Z_i$. Here $[n] \defeq \{1,2,\dots,n\}$. Note that the empirical distribution of the original dataset can be written in the form above by setting $\theta_i = e_i = 1$ for any $i$, i.e., $p_{Z;e} = \frac{1}{n} \sum_{i=1}^n e_i \delta_{Z_i}$. We will use $e$ to represent the $n$-vector with all entries being $1$. We use the Wasserstein distance between $p_{Z;\theta}$ and $p_{Z;e}$ to measure the discrepancy between the original and reweighted datasets.
%
%
To control for discrimination, we adopt the fairness constraints proposed by \citet{calmon2017optimized}, which are equivalent to imposing demographic parity over the original dataset. In our formulation, this translates to requiring the conditional distribution for all possible values of $D$ under the weights $\{\theta_i\}_{i\in[n]}$ to closely align with the marginal distribution over $Y$ in the original dataset, which we denote $p_{Y}$,
\begin{equation}\label{type1}
	J\left(p_{Z;\theta}(Y = y | D = d), p_{Y}(y)\right) \leq \epsilon, \ \forall \ d \in \mathcal{D}, y \in\calY
\end{equation}
where $J(\cdot, \cdot)$ denotes a distance function between scalars. We will use the shorthand $p_{Z;\theta}(y | d)$ for $p_{Z;\theta}(Y = y | D = d)$. This definition corresponds to the enforcing demographic parity by constraining the selection rates across groups $D=d$ to be equal to the overall selection rate.
However, unlike \citet{calmon2017optimized} who defined $J(p, q) $ as $ |\frac{p}{q}-1|$, we define $J$ as the subsequent symmetric probability ratio measure:
\begin{equation}\label{eq def of J(p,q)}
\begin{array}{c}
	J(p, q)=\max\left\{\frac{p}{q}-1,\frac{q}{p}-1\right\} \ .
\end{array}
\end{equation}
We believe our definition is more practical and theoretically sound because it is symmetric with respect to $p$ and $q$. We note that the two definitions are equivalent when $p > q$ and similar when $p$ is not much smaller than $q$.

Our proposed approach FairWASP finds \textit{integer} weights $\{\theta_i\}_{i\in[n]}$ via solving the following optimization problem:
\begin{equation}\label{pro general optimization framework}
    \begin{aligned}
    &\min_{\theta \in \calI^n \cap \Delta_n} \ \mathcal{W}_c(p_{Z;\theta},  p_{Z;e}) \\
    &\ \ \ \
     \ \ \text{s.t. } \ \ \ \ \ J\left(p_{Z;\theta}(y | d), p_{Y}(y)\right) \leq \epsilon, \ \forall \ d \in \mathcal{D}, y \in\calY,
    \end{aligned}
\end{equation}
where $\calI^n$ is the set of integer vectors in $\mathbb{R}^n$, and $\Delta_n$ is the set of valid weights $\{\theta \in \mathbb{R}^n_+: \sum_{i=1}^n \theta_i = n\}$. The use of integer weights can be understood simply as duplicating or eliminating samples in the original datasets.
This is in contrast with other approaches such as \citet{kamiran2012data} and \citet{bachem2017practical}, in which the sample-level weights are \textit{real-valued}. 
The problem of solving the optimal real-valued weights is instead as follows:
\begin{equation}\label{pro general optimization framework LP}
    \begin{aligned}
    & \min_{\theta \in\Delta_n} \ \mathcal{W}_c(p_{Z;\theta},  p_{Z;e}) \\
    & \ \ \ \text{s.t. } \ \ J\left(p_{Z;\theta}(y | d), p_{Y}(y)\right) \leq \epsilon, \ \forall \ d \in \mathcal{D}, y \in\calY.
    \end{aligned}
\end{equation}
Note that \eqref{pro general optimization framework LP} is in fact an LP relaxation of \eqref{pro general optimization framework}. 
In practice, using real-valued weights requires either (i) resampling each sample proportionally to its weight, which introduces statistical noise in the reweighted distribution, or (ii) including sample weights in the loss function during the learning process. Using integer weights, however, ensures the constructed dataset has exactly the optimal reweighted distribution (in the sense of \eqref{pro general optimization framework} and \eqref{pro general optimization framework LP}), and the reweighted dataset can be fed into any classification method, not just methods which accept sample weights. In addition, Theorem~\ref{thm integer constraints} and Lemma~\ref{lm optimal MIP is optimal LP} show that using integer weights achieves the optimal value of the objective in the optimization problem for real-valued weights, i.e., the optimal solution of \eqref{pro general optimization framework} is also an optimal solution for \eqref{pro general optimization framework LP}.

\section{Reformulations of the Optimization Problem} \label{sec: reformulation}

In this section, we provide a computationally tractable equivalent formulation of \eqref{pro general optimization framework}. In Step 1, we  reformulate \eqref{pro general optimization framework} as a mixed-integer program (MIP). However, directly solving this problem is infeasible due to its scale. In Step 2, we demonstrate that, through specific reformulations, the dual of the LP relaxation becomes more computationally manageable. In Step 3, we prove that the solution of the dual problem can lead to an optimal solution of \eqref{pro general optimization framework}.

\subsection{Step 1: Reformulating \eqref{pro general optimization framework} as a MIP}

First, we show that the constraint~\eqref{type1}  can be reformulated as linear constraints on $\theta$ of the form $A\theta \ge \mathbf{0}$. The conditional probability in constraint~\eqref{type1} can be rewritten as
$$
\begin{array}{c}p_{Z;\theta}(y | d) = \frac{\sum_{i\in[n]:d_i = d,y_i=y}\theta_i}{\sum_{i\in[n]:d_i = d}\theta_i}.
\end{array}$$

By substituting the definition of the distance $J(\cdot,\cdot)$ from \eqref{eq def of J(p,q)}, the fairness constraints equivalently become linear constraints on $\{\theta_i\}_{i=1}^n$ (via inverting a fractional linear transformation), taking the following form for all $d \in \mathcal{D}, y \in\calY$:
\begin{equation}\label{eq linear fairness constraints}
	\begin{array}{c}
        {\sum_{i\in[n]:d_i = d,y_i=y}\theta_i} 
		\le (1+\epsilon)\cdot  p_{Y}(y) \cdot
		{\sum_{i\in[n]:d_i = d}\theta_i} \ ,  \\
	{\sum_{i\in[n]:d_i = d,y_i=y}\theta_i} 
		\ge \frac{1}{1 + \epsilon}\cdot p_{Y}(y) \cdot
		{\sum_{i\in[n]:d_i = d}\theta_i} \ .
	\end{array}
\end{equation}
In total, \eqref{eq linear fairness constraints} defines $2|\calY||\calD|$ linear constraints on $\theta$ in the format of $A\theta \ge \mathbf{0}$, where $A$ is a  $2|\calY||\calD|$-row matrix\footnote{Note that when $Y$ is binary, e.g., $\calY = \{0, 1\}$, half of the linear constraints induced by \eqref{type1} are redundant and can be removed.}. 

Regarding the objective, the Wasserstein distance can be equivalently formulated as a linear program with $n^2$ variables \cite{peyre2019computational}.
Let $C \in \mathbb{R}^{n\times n}$ represent the matrix formed by the transportation costs, i.e., $C_{ij} = c(z_i,z_j)$. Then, according to definition \eqref{pro original wasserstrein distance}, the objective function $\calW_c(p_{Z;\theta}, p_{Z; e})$ is given by the optimal objective of the following problem:
\begin{equation}\label{pro wasserstrein objective}
\min_{P\in \mathbb{R}^{n\times n}} \langle C , P\rangle \ \text{ s.t. } P e  = e, \ P^\top e  = \theta, P \ge \mathbf{0}_{n\times n} \,
\end{equation}
where $\langle \cdot, \cdot \rangle$ is the Frobenius inner product and recall that $e=1$ is the vector of ones.
Hence, the integer-weight optimization problem in \eqref{pro general optimization framework} is equivalent to the following MIP:
\begin{equation}\label{MIP}
	\begin{aligned}
		\min_{\theta \in \mathbb{R}^n, P\in \mathbb{R}^{n\times n}} \ &  \langle C , P\rangle  
		\\
		\text{ s.t.} \ \ \ \ \ \ \ \  &  P e  = e, \ P^\top e  = \theta, P \ge \mathbf{0}_{n\times n} \\
		& \theta \in \calI^n \cap \Delta_n,\ A\theta \ge \mathbf{0} \ 
	\end{aligned}
\end{equation}
Similarly, the real-valued weights problem in \eqref{pro general optimization framework LP} is equivalent to the following LP:
\begin{equation}\label{LP}
	\begin{aligned}
		\min_{\theta \in \mathbb{R}^n, P\in \mathbb{R}^{n\times n}} \ &  \langle C , P\rangle  
		\\
		\text{ s.t.} \ \ \ \ \ \ \ \  &  P e  = e, \ P^\top e  = \theta, P \ge \mathbf{0}_{n\times n} \\
		& \theta \in \Delta_n,\ A\theta \ge \mathbf{0} \ 
	\end{aligned}
\end{equation}
Note that \eqref{LP} is actually also the LP relaxation of \eqref{MIP}. However, this reformulation is not yet practically useful as problem \eqref{LP} involves $O(n^2)$ variables, which poses a challenge for both conventional LP algorithms and state-of-the-art MIP methods, such as the LP based branch-and-bound methods \cite{gurobi}.



\subsection{Step 2: Dual Problem of the LP Relaxation}

In this step, we propose a solution of the LP relaxation \eqref{LP} by considering its dual problem. 

First, note that some constraints are currently redundant. For any feasible $(\theta,P)$, $\theta$ already lies in $\Delta_n$; given a feasible $P$, we have (i) $\theta = P^\top e$, (ii) $e^\top e  = n$ and (iii) $Pe  = e$, so it follows that $\theta^\top e  = e ^\top P e  = e^\top e  = n$.
Consequently, we can replace $\theta$ with $Pe$ and reformulate \eqref{LP} equivalently as:
\begin{equation}\tag{P}\label{pro general optimization model 2}
	\begin{aligned}
		\min_{P\in\mathbb{R}^{n\times n}}  \langle C , P\rangle  \
		\text{ s.t.} \  P  e  = e, \  P \ge \mathbf{0}_{n\times n}, \  A P^\top  e  \ge \mathbf{0} \ . 
	\end{aligned}
\end{equation}
Therefore, the optimal $\theta^\star$ of \eqref{LP} can be reconstructed from the optimal $P^\star$ of \eqref{pro general optimization model 2} using $\theta^\star = (P^\star)^\top  e $.

Second, we use a property of LP problems to reformulate \eqref{pro general optimization model 2}. When the feasible set of the LP problem \eqref{pro general optimization model 2} is nonempty and the optimal solution $P^\star$ exists, $P^\star$ is part of a saddle point of the saddle-point problem on the Lagrangian, 
\begin{equation}\tag{PD}\label{pro saddle point problem}
	\begin{aligned}
		\min_{P \in S_n} \max_{\lambda \in\mathbb{R}^m_+}\ &  L(P,\lambda) \defeq \langle C , P\rangle   - \lambda^\top  A P^\top  e  
	\end{aligned}
\end{equation}
where $S_n \defeq \{P\in\mathbb{R}^{n\times n}:P  e  = e, \  P \ge \mathbf{0}_{n\times n}\}$. 
Since $L(\cdot,\cdot)$ is bilinear, the minimax theorem \cite{du1995minimax} guarantees that \eqref{pro saddle point problem} is equivalent to $\max_{\lambda \in\mathbb{R}^m_+} 	\min_{P \in S_n}\   L(P,\lambda)$. This is then equal to the dual:
\begin{equation}\tag{D}\label{pro dual problem}
	\begin{aligned}
  \max_{\lambda \ge 0} -F(\lambda)  , \text{ where }F(\lambda)\defeq \max_{P \in S_n}   \Big\langle \bar{C} , P\Big\rangle  \,
	\end{aligned}
\end{equation}
where $\bar{C} = \sum_{j=1}^m \lambda_j  e a_j^\top - C$ and $a_j^\top$ is the $j$-th row of $A$. 

Unlike the problem in \eqref{LP}, the dual problem~\eqref{pro dual problem} can be directly solved, as we show in Lemma~\ref{lm first order oracles} below.

\begin{lemma}\label{lm first order oracles}
    For function $G(\bar{C})\defeq \max_{P \in S_n}   \langle \bar{C} , P\rangle $, it is a convex function of $\bar{C}$ in $\mathbb{R}^{n\times n}$. It has the following function value and subgradient.  For each $i\in [n]$, let $\bar{c}_{ij_i^\star}$             denote a largest component on the $i$-th row of $\bar{C}$, then $G(\bar{C}) = \sum_{i=1}^n  c_{ij_i^\star}$. Define the components of of $P^\star \in \mathbb{R}^{n\times n}$ as
		\begin{equation}\label{eq compute optimal P}
            \begin{array}{c}
		p_{ij} = \left\{
		\begin{array}{cc}
			0 \ ,& \text{ if }  j \neq j_i^\star \\
			1 \ , & \text{ if } j = j_i^\star
		\end{array}
		\right.
            \end{array}
		\end{equation}
		and then $P^\star \in \arg\max_{P \in S_n}   \langle \bar{C} , P\rangle $ and $P^\star \in   \partial G(\bar{C}) $.
\end{lemma}
\begin{proof}[Proof Sketch]
The proof directly uses the convexity of the maximium LP's optimal objective on the cost function. The problem can be divided into independent separate smaller LP on simplexes,  each having a closed-form maximizer.
\end{proof}

Due to the chain rule, Lemma \ref{lm first order oracles} shows that $F(\lambda)$ is convex and the function values and subgradients of $F(\lambda) = G( \sum_{j=1}^m \lambda_j  e a_j^\top- C)$ can be computed as well. This implies \eqref{pro dual problem} is equivalent to
\begin{equation}\tag{D-2}\label{eq equivalent dual}
\min_{\lambda \in \mathbb{R}^m_+} \ F(\lambda) \ ,
\end{equation}
whose objective function $F(\cdot)$ is a convex function of $\lambda$ (see Lemma \ref{lm first order oracles}). Here $m$ is the number of rows in matrix $A$, which as shown before is at most $2|\calY||\calD| \ll n$. Reformulation \eqref{eq equivalent dual} is important as it makes it possible to use methods that need only subgradients of the dual problem \eqref{pro dual problem}, such as the subgradient descent method and the cutting plane method \cite{nesterov2018lectures}, as we show below in Section~\ref{sec:cuttin_plane}.

Finally, we consider the implications for the uniqueness of the primal optimal solution $P^\star$.

\begin{assumption}\label{assump unique P star}
    The problem 
        $\min_{P \in S_n}   \langle    C - \sum_{j=1}^m \lambda_j^\star  e a_j^\top , P\rangle$ 
    has a unique minimizer for the optimal solution $\lambda^\star$ for \eqref{pro dual problem}. 
\end{assumption}

\begin{corollary}\label{corollary uniqueness dual}
    Under Assumption~\ref{assump unique P star}, the primal optimal solution $P^\star$ given by Lemma \ref{lm first order oracles} is the unique maximizer of $\max_{P \in S_n}   \langle  \bar{C} , P\rangle$.
\end{corollary}

\begin{proof}
    Once the optimal $\lambda^\star$ of \eqref{pro dual problem} is computed, using Assumption~\ref{assump unique P star} the optimal solution $P^\star$ of \eqref{pro general optimization model 2} then lies in $\arg\min_{P\in S_n}L(P,\lambda^\star)$, or equivalently $\arg\max_{P \in S_n}   \langle    \sum_{j=1}^m \lambda_j^\star  e a_j^\top  - C, P\rangle$.
\end{proof}

Assumption~\ref{assump unique P star} ensures there are no ties when calculating the row-wise max in the $\bar{C}$ matrix. Ties occur only for $\bar{C}$ in a set of measure zero, as the set of $\bar{C}$ such that  $\max_{P \in S_n}  \langle \bar{C} , P\rangle$ has multiple maximizers is the $\bar{C}$ with a row containing two or more largest components, which is of a strictly smaller dimension than the full space and thus zero measure. In practice, Assumption~\ref{assump unique P star} holds almost always due to rounding errors and the termination tolerance when computing the optimal solution $\lambda^\star$.

\subsection{Step 3: Using the Dual Solution to Solve the Original MIP}\label{ref:dual}

In this section, we show how to recover the optimal $P^\star$ and $\theta^\star$ of \eqref{LP} given the optimal solution $\lambda^\star$ of \eqref{pro dual problem}.
The following theorem demonstrates that the optimal solution $(\theta^\star,P^\star)$ of the LP \eqref{LP} recovered in this manner is also optimal for the MIP \eqref{MIP}.

\begin{theorem}\label{thm integer constraints}
    Let $\lambda^\star$ be an optimal dual solution of \eqref{pro dual problem} and let Assumption \ref{assump unique P star} hold. $P^\star$ is an optimal primal solution obtained through Lemma \ref{lm first order oracles} using the form of \eqref{eq compute optimal P}. Then it holds that $\theta^\star = (P^\star)^\top  e $ and $P^\star$ are  optimal solutions for both the LP \eqref{LP} and the MIP \eqref{MIP}.
\end{theorem}

\begin{proof}[Proof Sketch]
The proof uses the fact that problems~\eqref{LP} and \eqref{MIP} have the same objective function while the feasible set of \eqref{MIP} is smaller than that of \eqref{LP}, so if an optimal solution of \eqref{LP} is also feasible for \eqref{MIP}, then it is optimal for \eqref{MIP} as well.
\end{proof}

Theorem \ref{thm integer constraints} shows that once \eqref{LP} is solved by the dual problem \eqref{pro dual problem}, then \eqref{MIP} can be solved immediately. Finally, we can then also conclude that the solutions found by FairWASP are optimal even among real-valued weights.

\begin{lemma}\label{lm optimal MIP is optimal LP}
    When Assumption \ref{assump unique P star} holds, the optimal integer-weight solution of \eqref{pro general optimization framework} is as good as the optimal real-valued-weight solution of \eqref{pro general optimization framework LP}. 
\end{lemma}


\section{Cutting Plane Method for the Reformulated Problem}\label{sec:cuttin_plane}


The cutting plane method \cite{Khachiyan1980polynomial} is a class of methods for convex problems in settings where the \textit{separation oracle} is accessible. For any $\lambda \in \mathbb{R}^m$ in problem \eqref{eq equivalent dual}, 
a separation oracle is an operator  that returns a vector $g$ such that $g^\top \lambda \ge g^\top \lambda^\star$  for any $\lambda^\star \in \Lambda^\star$, where $ \Lambda^\star$ denotes the set of optimal solutions. The cutting plane method iteratively makes use of the separation oracle to restrict the feasible sets until convergence\footnote{In our case, convergence is achieved when the gap between the primal and dual problems is lower than a given tolerance.}. Algorithm~\ref{alg:cutting plane method} shows a pseudo-code breakdown of the cutting plane algorithm; variants of the cutting plane methods differ in the implementation of lines \ref{line:choose interior point} and \ref{line:nextE} (see \citealt{nesterov2018lectures} for more details).


%
\begin{algorithm}[t]
\caption{General Cutting Plane Method for \eqref{eq equivalent dual}}
\label{alg:cutting plane method}
\begin{algorithmic}[1]
\STATE Choose a bounded set $E_0$ containing an optimal solution
\FOR{$k$ from $0$ to $n$}
    \STATE Choose $\lambda^k$ from $E_k$ \label{line:choose interior point}
    \STATE Compute $g \in \mathbb{R}^m$ such that  \label{line:seperation oracle}
    $$
    g^\top \lambda^k \ge g^\top \lambda^\star \text{ for any }\lambda^\star \in  \Lambda^\star
    $$
    \STATE Choose $E_{k+1} \supseteq \{\lambda\in E_k: g^\top \lambda \le g^\top \lambda^k\}$ \label{line:nextE}
\ENDFOR
\end{algorithmic}
\end{algorithm}
For the problem~\eqref{eq equivalent dual}, a separation oracle (line 4 in Algorithm \ref{alg:cutting plane method}) can be obtained from the subgradients, which are efficiently accessible according to Lemma \ref{lm first order oracles}. Corollary~\ref{cor our complexity} below provides an analysis of both time and space complexity; see Supplementary Material A for more details on the separation oracle, implementation, and the proof.

\begin{corollary}\label{cor our complexity}
    With efficient computation and space management, the cutting plane method is able to solve the problem \eqref{eq equivalent dual} within 
    $ \tilde{O}
    \left(n^2 + |\calD|^2|\calY|^2 n \cdot \log(R/\eps)
    \right) $
    flops and $ {O}( n|\calD||\calY| ) $ space.\footnote{We use the notation $\tilde{O}(\cdot)$ to hide $m$, $n$, $|\calD|$, and $|\calY|$ in the logarithm function. Here $R$ denotes the maximum norm of the optimal solutions of \eqref{eq equivalent dual}}
\end{corollary}


\subsubsection{Comparison with Other LP Algorithms}

Table \ref{tbl LP algorithms} compares theoretical complexities and convergence rates of our cutting plane method implementation with the traditional simplex method and interior point method, as well as a recently proposed practical first-order method \cite{applegate2022faster,applegate2021practical} based on the primal-dual hybrid gradient (PDHG). 
Other LP algorithms, such as \citep{wang2022light}, are only for problems with special structures.
Note that the original LP problem \eqref{LP} has $O(n)$ constraints and $O(n^2)$ nonnegative variables, which scale badly with large values of $n$.
Table \ref{tbl LP algorithms} has already considered the benefit of sparse matrix multiplication; see Section~\ref{sec: experiments} for an empirical comparison of the computational efficiency of our cutting plane algorithm against existing commercial solvers, and Supplementary Material A for more details on the comparison. 
 
\begin{table}[t]
\begin{tabular}{ccccc}
\multirow{2}{*}{Method} & \multirow{2}{*}{Conv.} & \multicolumn{2}{c}{Time}        & \multirow{2}{*}{Space} \\ \cline{3-4}
                        &                        & Init.    & Per Iter.            &                        \\ \hline
Ours                    & Fast                   & $O(n^2)$ & $O(n|\calD||\calY|)$ & $O(n|\calD||\calY|)$   \\
Simplex                 & Slow                   & $O(n^2)$ & $O(n^3)$             & $O(n^2)$               \\
IPM                     & Fast                   & $O(n^2)$ & $O(n^3)$             & $O(n^2)$               \\
PDHG                    & Slow                   & $O(n^2)$ & $O(n^2)$             & $O(n^2)$              
\end{tabular}
\caption{Convergence speeds and complexity of different LP algorithms.}\label{tbl LP algorithms}
\end{table}

\section{FairWASP-PW: Extension to Pairwise Demographic Parity Constraints}\label{sec: fairwasp-pw}
As pointed out in \citet{calmon2017optimized}, demographic parity can be expressed in multiple equivalent forms.
In particular, we can rewrite \eqref{type1} to  constrain the selection rates to be (approximately) equal across groups $D=d$, rather than constraining them to be equal to the marginal distribution of $Y$ in the original dataset: 
\begin{equation}\label{type2}
	J\left(p_{Z;\theta}(y | d_1), p_{Z;\theta}(y | d_2)\right) \leq \epsilon, \ \forall d_1, d_2 \in \mathcal{D}, y \in\calY \ .
\end{equation}
This turns the optimization problem in \eqref{pro general optimization framework} into:
\begin{equation}\label{pro general optimization framework -- type II}
    \begin{aligned}
    &\min_{\theta \in \calI^n \cap \Delta_n} \ \mathcal{W}_c(p_{Z;\theta},  p_{Z;e}) \\
    & \text{s.t. } J\left(p_{Z;\theta}(y | d_1), p_{Z;\theta}(y | d_2)\right) \leq \epsilon, \ \forall d_1, d_2 \in \mathcal{D}, y \in\calY \ .
    \end{aligned}
\end{equation}
%
%
This section introduces FairWASP-PW, which extends FairWASP to constraints~\eqref{type2}. We show how to solve \eqref{pro general optimization framework -- type II} by (i) pointing out a connection between constraints \eqref{type1} and \eqref{type2}, (ii) reformulating problem \eqref{pro general optimization framework -- type II} and connecting it to problem~\eqref{pro general optimization framework} ,and (iii) solving \eqref{pro general optimization framework -- type II} via zero-th order optimization.




\subsubsection{(1) Connection between constraints 
 \eqref{type1} and \eqref{type2}.}

For any $|\calY|$-vector $t \in [0,1]^{|\calY|}$ denoting the marginal distribution $p_Y(y)$ in \eqref{type1}, let $\Theta_{\epsilon;t}$ denote the $\theta$ that satisfies the fairness constraint~\eqref{type1}:
\begin{equation}\label{def type 1 theta}
\Theta_{\epsilon;t} \defeq  \{ \theta\in\Delta_n: J\left(p_{Z;\theta}(y | d), t\right) \leq \epsilon, \  \forall \ d \in \mathcal{D}, y \in\calY \}.
\end{equation}
Hence, the feasible sets of~\eqref{pro general optimization framework} under constraint~\eqref{type1}  is $\calI^n \cap  \Theta_{\epsilon;\bar{t}}$, where $\bar{t}_y= p_Y(y)$.
%
As for the feasible set of problem~\eqref{pro general optimization framework -- type II} under constraint~\eqref{type2}, define
\begin{equation}\label{def type 2 theta}
     \Theta_{\epsilon} \defeq
\left\{ \theta\in\Delta_n: \left.\begin{array}{c} 
 J\left(p_{Z;\theta}(y | d_1), p_{Z;\theta}(y | d_2)\right) \leq \epsilon,  \\
\forall \ d_1,d_2 \in \mathcal{D}, y \in\calY \end{array} \right.\right\}, 
\end{equation}
obtaining $\calI^n \cap  \Theta_{\epsilon}$ as the corresponding feasible set.

The following lemma shows how the feasible set for problem~\eqref{pro general optimization framework} is a subset of problem~\eqref{pro general optimization framework -- type II}'s feasible set. More specifically, $\Theta_{\epsilon}$ is equal to the union of $\Theta_{\bar{\epsilon};\bar{t}}$ for all $\bar{t} \in [0,1]^\mathcal{Y}$ and a certain $\bar{\epsilon}$.


\begin{lemma}\label{lm bridge type12}
    Let $\Theta_{\epsilon; t}$ and $\Theta_{\epsilon}$ be defined as \eqref{def type 1 theta} and \eqref{def type 2 theta}, then it holds that for any $\epsilon \in [0,1)$, $
        \Theta_{\epsilon} = \bigcup_{t \in [0,1]^{\calY}}\Theta_{\bar{\epsilon}; t} $, 
    in which $\bar{\epsilon} = \sqrt{1+\epsilon} - 1$.
\end{lemma}

\begin{proof}[Proof Sketch]
    Inclusion from both sides can be shown via constructing an element of each set respectively.
\end{proof}

Note that $\Theta_{\epsilon}$ is not convex, as the union of convex sets is not necessarily convex, making problem~\eqref{pro general optimization framework -- type II} not convex.

\subsubsection{(2) Reformulation of Problem~\eqref{pro general optimization framework -- type II}.}

Using Lemma~\ref{lm bridge type12}, we can rewrite problem~\eqref{pro general optimization framework -- type II} as:
\begin{equation}\label{pro type2-1}
		\min_{\theta \in\mathbb{R}^n} \ \calW_c(p_{Z;\theta},  p_{Z;e}) 
		\ \ \text{s.t.} \   \theta \in  \calI^n \cap (\cup_{t \in [0,1]^{\calY}}\Theta_{\bar{\epsilon}; t}) \ ,  
\end{equation}
which is in turn equivalent to the following problem that simultaneously optimizes over $t$:
\begin{equation}\label{pro type2-2}
		\min_{\theta \in\mathbb{R}^n , t \in [0,1]^{\calY}} \  \calW_c(p_{Z;\theta},  p_{Z;e}) \ \
		\text{s.t.}  \    \theta \in   \calI^n \cap \Theta_{\bar{\epsilon}; t} \ . 
\end{equation}
Compared with problem~\eqref{pro general optimization framework}, problem \eqref{pro type2-2} has $t$ as part of the decision variables with $p_Y(y) = t$ and $\epsilon = \bar{\epsilon}$. In other words, if we denote $H_I(t;\bar{\epsilon})$ the optimal objective values for the MIP in \eqref{pro general optimization framework}, then \eqref{pro type2-2} is equal to:
\begin{equation}\label{pro type2-3}
    \min_{t\in[0,1]^{\calY}} H_I(t;\bar{\epsilon}).
\end{equation}
Once the optimal $t^\star$ of \eqref{pro type2-3} is obtained, fixing $t=t^\star$ in \eqref{pro type2-2} and optimizing over $\theta$ yields the optimal weights $\theta^\star$.


\subsubsection{(3) Zero-th Order Optimization Methods for \eqref{pro type2-3}.}
We propose to employ zero-th order optimization methods for the minimization problem in \eqref{pro type2-3}. In our setting, this is a particularly efficient choice, as:

\begin{itemize}
    \item the value of $H_I(t;\bar{\epsilon})$ can be computed via the dual problem \eqref{pro dual problem}, as discussed above. Since the cost matrix remains unchanged, after solving \eqref{pro dual problem} for the first time, the complexity of solving the problem again with any different $t$ is only $\tilde{O}(n|\calY|^2|\calD|^2 \log(R/\eps))$;
    \item the problem in \eqref{pro type2-3} is of dimension $|\calY|$, so low-dimensional, with only unit box constraints.
\end{itemize}
%
%
Many methods have shown fast convergence to stationary points for very-low-dimensional problems in practice, such as the multi-dimension golden search method \cite{chang2009n} and the Nelder-Mead method \cite{gao2012implementing}. We opt for the latter in our implementation.

\subsubsection{Optimality of Integer Weights.} Note that once the optimal $t^\star$ of \eqref{pro type2-3} is obtained, the problem \eqref{pro type2-2}  with $t$ fixed as $t^\star$ is an instance of problem~\eqref{pro general optimization framework} with $p_Y(y) = t^\star$ and $\epsilon = \bar{\epsilon}$. Hence, according to Theorem~\ref{thm integer constraints} and Lemma~\ref{lm optimal MIP is optimal LP}, the optimality of integer weights also carries over to problem~\eqref{pro general optimization framework -- type II}.


\section{Experiments} \label{sec: experiments}

In this section, we use a synthetic dataset to provide an efficiency analysis of FairWASP against established state-of-the-art commercial solvers. In addition, we show on various real datasets how FairWASP achieves competitive performance compared to existing methods in reducing disparities while preserving accuracy in downstream classification settings.


\subsubsection{Synthetic dataset}

We generate a synthetic dataset in which one of the features is highly correlated with the protected variable $D$, in order to induce a dependency of the outcome on $D$. We generate a 
binary protected variable $D=\{0, 1\}$ and features $\mathbf{X} = [X_1, X_2] \in \mathbb{R}^2$, such that $X_1$ is dependent on the value of $D$ and $X_2$ is not. More specifically, $X_1 \sim \mathcal{U}[0,10] \cdot \mathbb{I}(D=1)$, where $\mathcal{U}$ indicates the uniform distribution and $\mathbb{I}$ the indicator function, so that $X_1 = 0$ if $D=0$, and $X_2 \sim \mathcal{N}(0,25)$. The outcome $Y$ is binary and defined as $Y= \mathbb{I}(X_1 + X_2 + \eps > m_X)$, where $m_X = \mathbb{E}(X_1 + X_2)$ and $\eps \sim \mathcal{N}(0,1)$ is random noise.

Figure \ref{fig:runtime} compares the runtime of the FairWASP and commercial solvers, Gurobi and Mosek, in solving problem~\eqref{LP} for the synthetic data with different number of samples $n$ (mean and standard deviation over 5 independent trails, with $n$ doubling from $n=100$ up to $n=12,800$). The runtime limit for all methods is set to 1 hour, which both commercial solvers exceed when $n > 10,000$. In contrast, FairWASP has a significantly faster runtime than commercial solvers, solving all optimization problems within $5$ seconds. As the commercial solvers are run with default settings, we show that the solutions found by FairWASP are comparable to the commercial solver solutions in Supplementary Material C.
\begin{figure}[t]
    \centering
    \includegraphics[width=1\columnwidth]{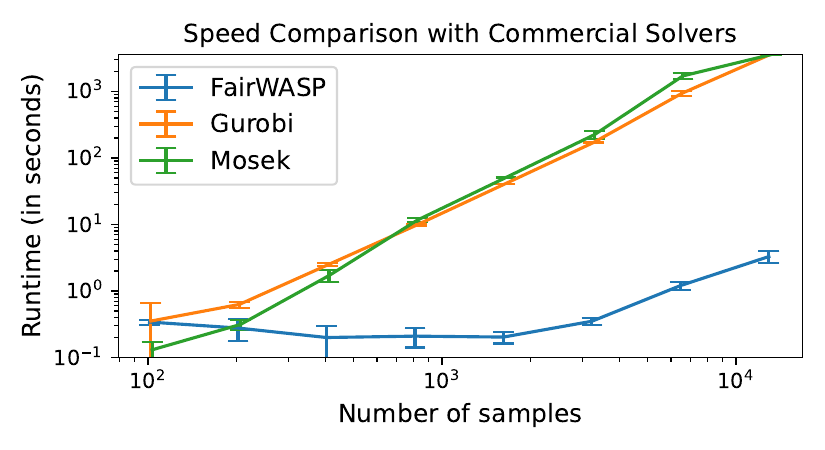}
    \caption{Speed comparison with commercial solvers. FairWASP has significantly better runtime and scalability.}
    \label{fig:runtime}
\end{figure}

\begin{figure*}[!ht]
    \centering
    \includegraphics[width=0.49\textwidth]{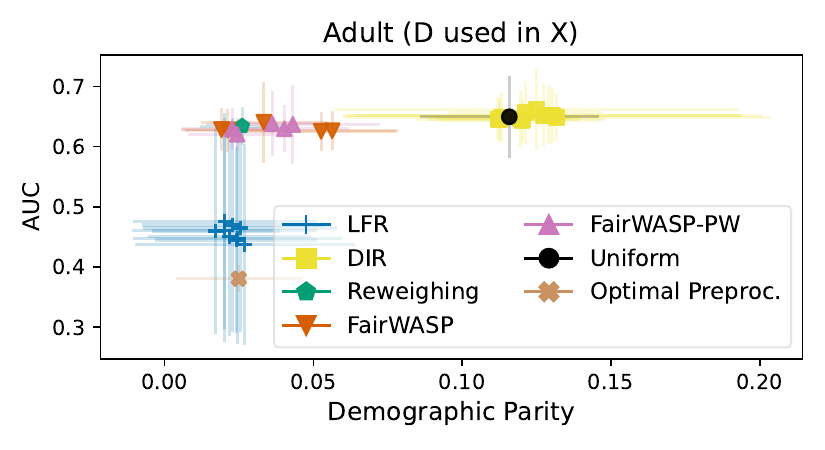} \hfill
    \includegraphics[width=0.49\textwidth]{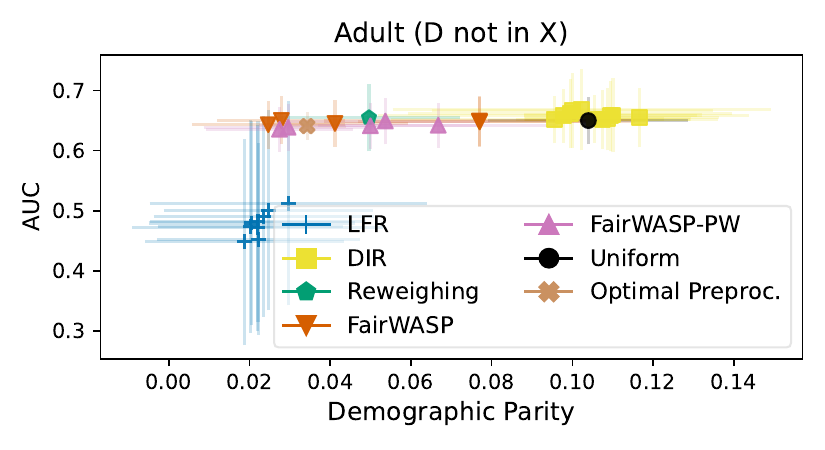} \\
    \includegraphics[width=0.49\textwidth]{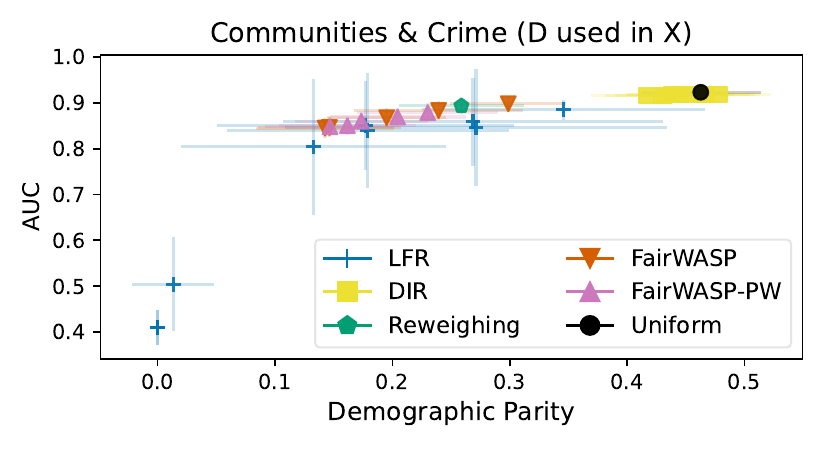} \hfill
    \includegraphics[width=0.49\textwidth]{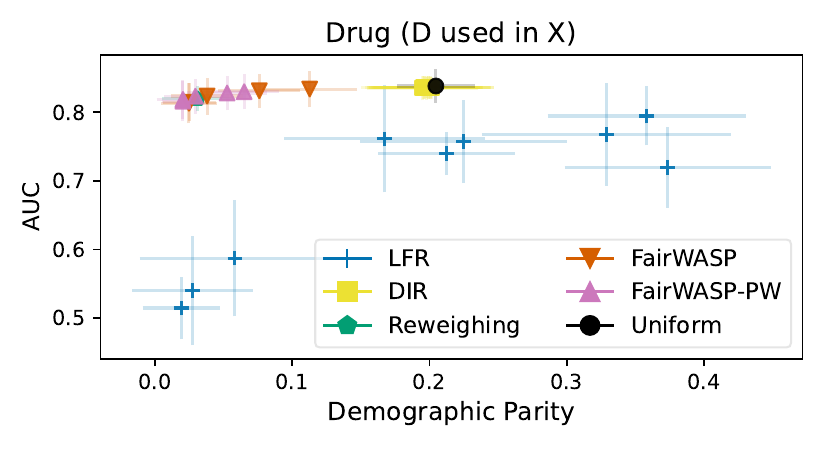} \\
    \caption{Downstream fairness-utility tradeoff, indicated by the demographic disparity and downstream classifier area under the curve (AUC). The x-axis refers to the absolute difference in the mean classifier outcome for the two groups, with a value of 0 corresponding to perfect demographic parity. Points and error bars correspond to averages plus/minus one standard deviation, computed over 10 different train/test split. FairWASP and FairWASP-PW consistently provide one of the best tradeoffs, significantly improving over using the original dataset as-is. See text and Supplementary Material C for more details.}
    \label{fig:resultsreal}
\end{figure*}

\subsubsection{Real Datasets}

We consider the following four real datasets widely used in the fairness literature \cite{fabris2022algorithmic}: (i) the  Adult dataset \cite{misc_adult_2}, (ii) the Drug dataset \cite{fehrman2017five}, (iii) the Communities and Crime dataset \cite{misc_communities_and_crime_183} and (iv) the German Credit dataset \cite{misc_statlog_(german_credit_data)_144}. We compare the performance of FairWASP and FairWASP-PW with the following existing pre-processing approaches:
\begin{itemize}
    \item \textit{DisparateImpactRemover} (DIR, \citealt{feldman2015certifying}), which transforms feature values in a rank-preserving fashion,
    \item \textit{Learning fair representations} (LFR, \citealt{zemel2013learning}), which identifies a latent representation uncorrelated with the protected attributes,
    \item  \textit{Reweighing} \cite{kamiran2012data}, which weights each sample according to the respective $(D, Y)$ values,
    \item \textit{Optimized pre-processing} \cite{calmon2017optimized}, which learns a probabilistic transformation to be applied to the dataset so that it satisfies group fairness, individual distortion and fidelity constraints.
\end{itemize}

We also include the \textit{Uniform} approach, which corresponds to the baseline of training on the dataset as-is. In all methods, the pre-processed dataset (or the dataset with no pre-processing, for the \textit{Uniform} approach) is used to train a multi-layer perceptron (MLP) classifier with one hidden layer with 20 nodes and ReLu activation function. Figure~\ref{fig:resultsreal} shows the fairness-utility tradeoff, indicated by the demographic disparity (defined in the caption) and the classifier AUC, for the Adult dataset (top row), Communities \& Crime dataset (bottom left) and Drug dataset (bottom right). We include the settings in which the protected variable $D$ is included among the features $X$ or not; the latter corresponds to the realistic scenario in, e.g., loan credit approvals, in which the US Equal Credit Opportunity Act of 1974\footnote{\url{https://www.law.cornell.edu/uscode/text/15/1691}} prohibits the use of such protected features. In all settings, FairWASP and FairWASP-PW are consistently part of the so-called ``Pareto frontier'' of the fairness utility tradeoff \cite{ge2022toward}, meaning they usually achieve either the best or among the best fairness-utility tradeoffs (closest to the $(0,1)$ in the top left corner), significantly improving over the empirical distribution (the \emph{Uniform} approach). See Supplementary Material C for more details on datasets, hyper-parameter settings and downstream fairness-accuracy tradeoffs for all datasets.

\section{Conclusions}

We propose FairWASP, a novel pre-processing algorithm that returns sample-level weights for a classification dataset without modifying the training data. FairWASP solves an optimization problem that minimizes the Wasserstein distance between the original and the reweighted dataset while satisfying demographic parity constraints. We solve the optimization problem by reformulating it as a mixed-integer program, for which we propose a highly efficient algorithm that we show to be significantly faster than existing commercial solvers. FairWASP returns integer weights, which we show to be optimal, and hence which can be understood as eliminating or duplicating existing samples, making it compatible with any downstream classification algorithm. We empirically show how FairWASP achieves competitive performance with existing pre-processing methods in reducing discrimination while maintaining accuracy in downstream classification tasks.


For future work, we would like to (i) characterize the finite sample properties of FairWASP for the downstream fairness-utility tradeoff, (ii) explore the downstream effect of using different  distances in the calculation of the cost matrix $C$, such as the Wasserstein transform \cite{memoli19awasstransf}, and (iii) extend the proposed optimization framework to non-linear fairness constraints as well as to general LPs and MIPs with similar structures.

\section*{Acknowledgments}
Some of this work was performed while the first author was at JPMorgan Chase \& Co. This paper was prepared for informational purposes by the Artificial Intelligence Research group of JPMorgan Chase \& Co. and its affiliates (``J.P. Morgan''), and is not a product of the Research Department of J.P. Morgan. J.P. Morgan makes no representation and warranty whatsoever and disclaims all liability, for the completeness, accuracy or reliability of the information contained herein. This document is not intended as investment research or investment advice, or a recommendation, offer or solicitation for the purchase or sale of any security, financial instrument, financial product or service, or to be used in any way for evaluating the merits of participating in any transaction, and shall not constitute a solicitation under any jurisdiction or to any person, if such solicitation under such jurisdiction or to such person would be unlawful.

\newpage

\appendix

\begin{center}
\LARGE{\textbf{SUPPLEMENTARY MATERIALS}}
\end{center}

\section*{A: More Details on Cutting Plane Method and Comparison with Other LP Algorithms}

The cutting plane method is a class of methods for convex problems when having access to their separation oracles. We now show that, in Algorithm~\ref{alg:cutting plane method} the separation oracle in line~\ref{line:seperation oracle} is computationally inexpensive as it has the same complexity as computing the subgradient.
The mathematical justification for this fact is as follows. For any $\bar{\lambda} \in \mathbb{R}^m_+$,  suppose that $g$ is a subgradient of $F(\cdot)$ at $\bar{\lambda}$ and then  due to the definition of the subgradients on convex problems,
$g^\top (\lambda^\star - \bar{\lambda} ) \le F(\lambda^\star) - F(\bar{\lambda}) \le 0$ for any $\lambda^\star \in \Lambda^\star$.
For any $\bar{\lambda} \notin \mathbb{R}^m_+$, there exists $j \in [m]$ such that $\bar{\lambda}_j < 0$, and then let $g$  be defined such that $g_j = -1$ and $g_i = 0$ for any $i \neq j$. For this $g$, $g^\top (\lambda^\star - \bar{\lambda} ) = -\lambda_j^\star + \bar{\lambda}_j \le 0$. 

Overall,  in line \ref{line:seperation oracle} of Algorithm \ref{alg:cutting plane method}, the complexity of obtaining a separation oracle for \eqref{eq equivalent dual} requires obtaining a subgradient of  $\lambda^k$ or finding a negative entry of $\lambda^k$. In the subsections below, we show how to efficiently compute subgradients and how  Algorithm \ref{alg:cutting plane method} applies on \eqref{eq equivalent dual}.


\subsubsection{Efficient computation of subgradients.} 

The subgradients and function values of problem \eqref{eq equivalent dual} can be very efficiently computed. 

\begin{lemmaS}\label{lm complexity of first order oracle}
    Let $[n]$ be divided into at most $|\calD||\calY|$ groups $G_1, G_2,\dots, G_{|\calD||\calY|}$ so that in each group, all individual data have the same $D_i$ and $Y_i$. Efficient computation of subgradients requires first computing $\arg\min_{l\in G_j}C_{ij}$ for each row $i$ and any group $j$, which requires $O(n^2)$ flops and $O(n|\calY||\calD|)$ space. After that, computing the subgradient and function value of $F(\lambda)$ for any $\lambda$ requires only $O(n|\calY||\calD|)$ flops and $O(n|\calY||\calD|)$ space 
\end{lemmaS}

\begin{proof}
    Due to the chain rule, let a subgradient of $G(\cdot)$ at $\sum_{j=1}^m \lambda_j e a_j^\top  -C $ be $P^\star$, then the subgradient of $F(\cdot)$ at $\lambda$ is $\left(\langle  e  a_1^\top , P^\star \rangle, \dots,\langle  e  a_m^\top , P^\star \rangle \right)$.  And Lemma \ref{lm first order oracles} shows that the function value of $F(\cdot)$ at $\lambda$  is equal to $\langle C - \sum_{j=1}^m \lambda_j e a_j^\top  , P^\star \rangle$. 
    
    According to Lemma \ref{lm first order oracles}, 
    the computation of the subgradients and function values requires the largest component in each row of the matrix $C_{\lambda} \defeq \sum_{j=1}^m \lambda_j e a_j^\top-C$.
    Here the rows of matrix $A$ corresponds to constraints in   \eqref{eq linear fairness constraints}. If the column indices of two components at the same row of $A$ correspond to the same class  and label, it could be observed from \eqref{eq linear fairness constraints} that the values of these components are the same.  In other words, there exists at most $|\calY||\calD|$ groups, as a partition of $[n]$, so that in any row of $\sum_{j=1}^m \lambda_j e a_j^\top$, the components in the same group are the same. Therefore, let the groups be $G_1, G_2,\dots, G_{|\calD||\calY|}$ and let 
    $c_{ik_l^i}$ be the minimizer of $C$'s components in the $i$-th row and the $l$-th group, $\arg\min_{k\in G_l}C_{ik}$,
    and then a maximizer of the $i$-th row of $C_{\lambda}$ could be found from the components with column index $k_l^i$ for a certain $l \in [|\calD||\calY|]$. 
    This observation gives an efficient computation of gradients and function values via only using the components with column indices  $k_l^i$ for $l \in [|\calD||\calY|]$ in each row $i$ of $C_{\lambda}$. 

    To begin with, computing the whole matrix $C$  requires $O(n^2)$ flops so computing $\arg\min_{l\in G_j}C_{ij}$ for each row $i$ and any group $j$ requires at most $O(n^2)$ flops. However, these optimization problems are separable, so if computing serially, they only use $O(n|\calY||\calD|)$ space.

    After that, computing the subgradient $P^\star$ needs only the maximizer of each row of $\sum_{j=1}^m \lambda_j e a_j^\top  -C$. 
    The maximizers occur in the components in each row $i$ with column indices $k_l^i$ for $l \in [|\calD||\calY|]$, so computing these components and finding the smallest one require $O(n|\calY||\calD|)$ flops and at most $O(n)$ space. The obtained $P^\star$ is highly sparse, with only $n$ nonzero components. 

    Next, the function value of $F(\cdot)$ at $\lambda$  is $\langle C - \sum_{j=1}^m \lambda_j e a_j^\top  , P^\star \rangle$, whose computation requires only $O(mn)$ flops thanks to the sparsity of $P^\star$. Since the number of linear constraints $m$ is no more than $2|\calY||\calD|$, computing the function value requires 
    only $O(n|\calY||\calD|)$ flops. As for the space, since $C$ could be replaced with only the components in each row $i$ with column indices $k_l^i$, it uses at most $O(n|\calY||\calD|)$ space.

    As for the subgradient, we have shown that it is $\left(\langle  e  a_1^\top , P^\star \rangle, \dots,\langle  e  a_m^\top , P^\star \rangle \right)$. Due to the sparsity of $P^\star$, computing this $m$ component-wise products only needs $O(mn)$ flops, which is still $O(n|\calY||\calD|)$. Similarly, it uses at most $O(n|\calY||\calD|)$ space.
    Overall, computing a subgradient also requires  only $O(n|\calY||\calD|)$ flops and $O(n|\calY||\calD|)$ space.
\end{proof}

Furthermore, notice that $|\calD|$ is the number of demographic classes and $|\calY|$ is the number of possible outcomes. Both of them are much smaller than $n$. Once the minimizers $\arg\min_{l\in G_j}C_{ij}$ for each row $i$ and group $j$ are computed beforehand, Lemma \ref{lm complexity of first order oracle} shows that we actually have very easily accessible subgradients and function values for \eqref{eq equivalent dual}.

\subsubsection{Cutting plane method for solving the dual problem.}

Given that we have shown \eqref{eq equivalent dual} is a low-dimension convex  program with subgradient oracles, there exists many well-establish algorithms. 
A well-known example is the ellipsoid method. In the same settings as Algorithm~\ref{alg:cutting plane method}, the ellipsoid method would generate ellipsoid $E_k$ in line \ref{line:nextE} and choose the centers of these ellipsoids as $\lambda^k$ in \ref{line:choose interior point}.
Suppose that the norm of the optimal $\lambda^\star$ is bounded by $R$ and then directly applying the ellipsoid method \cite{nesterov2018lectures} on \eqref{eq equivalent dual} requires $O((m^2\cdot \operatorname{SO} + m^4 )\cdot \log(mR/\eps))$  flops to get a solution whose distance to optimal solutions smaller than $\eps$, where $\operatorname{SO}$ denotes the complexity of a separation oracle. 
To the best of our knowledge, the cutting plane method with the best theoretical complexity is given by \citet{jiang2020improved}, who proposed an improved cutting plane method that only needs $O((m\cdot \operatorname{SO} + m^2 )\cdot \log(mR/\eps))$ flops. Note that here $m$ is at most $2|\calD||\calY|$ and far smaller than $n$. Overall, Corollary \ref{cor our complexity} follows by it.

Finally, for our implementation, we use the analytic center cutting plane method proposed by \citet{atkinson1995cutting} to solve \eqref{eq equivalent dual} and follow the improved implementation introduced by \citet{Boydanalyticcenter}.

\subsection{Details of the comparison}

The state-of-art commercial LP solvers, such as Gurobi \cite{gurobi} and Mosek \cite{mosek}, use the well-established simplex and interior point method. The simplex method is known to have no good computation guarantees and may suffer from slow convergence rates in some cases \cite{klee1972good}. Moreover, the simplex method solves an $n$ by $n$ linear system in each iteration. Although commercial solvers can efficiently utilize the sparsity, it is still not suitable for large-scale problems.
The interior point method enjoys faster convergence rates $\tilde{O}(nL\log(1/\eps))$\footnote{Here $L$ denotes the bit complexity of data.} \cite{renegar2001mathematical} (usually  much faster in practice) but each iteration requires constructing and solving an $n$ by $n$ linear system, which can be very time-consuming. If using standard matrix factorization, each iteration requires  $O(n^3)$ flops even when using the sparsity. It should be mentioned that better theoretical complexity results have been recently presented in \cite{lee2014path,peng2021solving} et al. by changing the barrier function and solving the linear system more efficiently, but these results are  far from practice and the state-of-art practical performance achieved by commercial solvers still use the traditional interior point methods. 
Recently, \citet{applegate2022faster,applegate2021practical} propose to use the primal-dual hybrid gradient (PDHG) to solve the large-scale LP problems. The PDHG avoids expensive matrix factorizations and instead does matrix-vector products in each iteration,  but it converges much slower than the interior point method in both theory and practice.   Table \ref{tbl LP algorithms} shows a comparison with the simplex method (Simplex), the interior point method (IPM) and the PDHG's theoretical convergence rates, and per iteration complexities in solving \eqref{LP}.

It should be further mentioned that all the other LP algorithms in Table \ref{tbl LP algorithms} could only solve the LP \eqref{LP} while our method  also solves the MIP \eqref{MIP}. State-of-art commercial solvers such as Gurobi use LP based branch-and-bound algorithms to solve MIPs so it is usually much slower than solving only the LP relaxation.


\section*{B: Other Theoretical Proofs}

\begin{proof}[Proof of Lemma \ref{lm first order oracles}]
     To begin with, since the optimal cost of a minimization LP is a concave function of the cost vector (Theorem 5.3 of \cite{bertsimas1997introduction}), $G(\bar{C})$ is a convex function of $\bar{C}$. 

    Next, the problem $\max_{P \in S_n}   \langle \bar{C} , P\rangle $ is equivalent to solving  $n$ separate smaller optimization problems
    \begin{equation}\label{pro small problem}
    \max \ c_{i}^\top p_i \ \ \text{s.t.} \  e ^\top p = 1, \ p \ge 0, 
    \end{equation}
    $\text{for } i \in [n]$,  where $c_{i}$ is the vector of the $i$-th row of $\bar{C}$ and $p_i$ denotes the $i$-th row of $P$. The optimal solution of \eqref{pro small problem} is then given by
    		$$
		p_{ij} = \left\{
		\begin{array}{cc}
			0 \ ,& \text{ if }  j \neq j_i^\star \\
			1 \ , & \text{ if } j = j_i^\star
		\end{array}
		\right. \text{ for } j \in [n] \ ,
		$$
  and consequently, the optimal objective of $G(\bar{C})$ is $\sum_{i=1}^n  c_{ij_i^\star}$.

  Let the optimal solution of $\arg\max_{P \in S_n}   \langle \bar{C}_1 , P\rangle $ be $P^\star$. For any $\bar{C}_2 \neq \bar{C}_1$, it holds that 
  $$
  \langle \bar{C}_2, P^\star \rangle \le G(\bar{C}_2)
  $$
  because $P^\star$ is feasible but not necessarily optimal for $\arg\max_{P \in S_n}   \langle \bar{C}_2 , P\rangle $. Note that 
  $$
  \langle \bar{C}_2, P^\star \rangle = G(\bar{C}_1) + \langle P^\star ,  \bar{C}_2 - \bar{C}_1 \rangle
  $$
  hence $P^\star$ serves as a subgradient of $G(\cdot)$ at $\bar{C}_1$.
\end{proof}

\begin{proof}[Proof of Theorem \ref{thm integer constraints}]
    As demonstrated earlier, $P^\star$ is an optimal solution for \eqref{pro general optimization model 2}, leading to the recovery of $\theta^\star= (P^\star)^\top  e $ as the optimal solution of \eqref{LP}. Therefore, $(\theta^\star, P^\star)$ solves \eqref{LP}. Owing to the computation of $P^\star$ in \eqref{eq compute optimal P},  components of $P^\star$ are either $0$ or $1$. Given that $\theta^\star = (P^\star)^\top  e $, each component of $\theta^\star$ must be an integer, indicating that $(\theta^\star, P^\star)$ is also feasible for \eqref{MIP}. Because the feasible set of the LP relaxation contains the feasible set of the original MIP, once an optimal solution of the LP relaxation satisfies the integer constraints, it also lies in the feasible set of the original MIP, and thus must be the optimal solution for the original MIP as well. Since \eqref{LP} is an LP relaxation for \eqref{MIP} and $\theta^\star$ also satisifies the integer constraints, $(\theta^\star, P^\star)$  must also be an optimal solution for \eqref{MIP}.
\end{proof}

\begin{proof}[Proof of Lemma \ref{lm bridge type12}]
    First of all, for any $t \in [0,1]^{\calY}$ and any   
    $\theta \in \Theta_{\bar{\epsilon}; t}$, we first show $\theta$ is also in  $\Theta_{\epsilon}$.
    It holds that for any $ d_1,d_2 \in \mathcal{D}$ and $y \in\calY$, 
    \begin{align}
        \frac{p_{Z;\theta}(y | d_1)}{p_{Z;\theta}(y | d_2)} - 1 &=
        \frac{p_{Z;\theta}(y | d_1)}{t} \cdot \frac{t}{p_{Z;\theta}(y | d_2)} - 1 \notag \\
        & \le (1+\bar{\epsilon})^2  -1 = \epsilon \ .
    \end{align}
    Symmetrically, $\frac{p_{Z;\theta}(y | d_2)}{p_{Z;\theta}(y | d_1)} - 1  \le \epsilon$. Therefore, $J\left(p_{Z;\theta}(y | d_1), p_{Z;\theta}(y | d_2)\right) \le \epsilon$ for any $y_1,y_2\in \calY$. This means $\theta$ is also in  $\Theta_{\epsilon}$, and then $ \bigcup_{t \in [0,1]^{\calY}}\Theta_{\bar{\epsilon}; t}\subseteq\Theta_{\epsilon} $.

    Next, for any $\theta \in \Theta_{\epsilon}$, we construct $t$ such that $\theta \in \Theta_{\bar{\epsilon}; t}$. 
    Due to the definition of \eqref{type2}, we have 
    \begin{equation}
        \begin{aligned}
            &\Theta_{\epsilon} \ = \Bigg\{ \theta\in\Delta_n: \frac{p_{Z;\theta}(y | d_1)}{p_{Z;\theta}(y | d_2)} - 1  \le \epsilon\text{ and} \\
            &\hspace{2cm} \frac{p_{Z;\theta}(y | d_2)}{p_{Z;\theta}(y | d_1)} - 1  \le \epsilon, \  \forall \ d_1,d_2 \in \mathcal{D}, y \in\calY \Bigg\} \\
            & \ = \left\{ \theta\in\Delta_n: \frac{\max\{p_{Z;\theta}(y | d):d\in\calD\}}{\min\{p_{Z;\theta}(y | d):d\in\calD\}} - 1  \le \epsilon, \  \forall \ y \in\calY \right\} \ . 
        \end{aligned}
    \end{equation}
    Next, we construct $t \in [0,1]^{\calY}$ as follows,
    $$
    t_y \defeq \sqrt{\max\{p_{Z;\theta}(y | d):d\in\calD\}\cdot \min\{p_{Z;\theta}(y | d):d\in\calD\}}
    $$
    and then with this $t$,  for any $y \in \calY$, 
    \begin{align}\label{eq max t}
    & \frac{\max\{p_{Z;\theta}(y | d):d\in\calD\}}{t_y} - 1 = \notag \\
    &\hspace{2.5cm} =\sqrt{ \frac{\max\{p_{Z;\theta}(y | d):d\in\calD\}}{\min\{p_{Z;\theta}(y | d):d\in\calD\}}} - 1 \notag \\
    &\hspace{2.5cm} \le \sqrt{1+\epsilon} - 1 = \bar{\epsilon}.
    \end{align}
    And similarly, for any $y \in \calY$,
    \begin{align}\label{eq min t}
    &\frac{t_y}{\min\{p_{Z;\theta}(y | d):d\in\calD\}} -1 = \notag \\
    &\hspace{2.5cm} =\sqrt{ \frac{\max\{p_{Z;\theta}(y | d):d\in\calD\}}{\min\{p_{Z;\theta}(y | d):d\in\calD\}}} - 1 \notag \\
    &\hspace{2.5cm} \le \sqrt{1+\epsilon} - 1 = \bar{\epsilon}.
    \end{align}
    The above \eqref{eq max t} and \eqref{eq min t} indicate that $J(p_{Z;\theta}(y | d),t_y)\le\bar{\epsilon}$ for any $d \in \calD$ and $y\in \calY$. This means $\theta$ is also in  $\Theta_{\bar{\epsilon};t}$, and then $\Theta_{\epsilon}  \subseteq\bigcup_{t \in [0,1]^{\calY}}\Theta_{\bar{\epsilon}; t}$. 
    
    Together with the previous result $ \bigcup_{t \in [0,1]^{\calY}}\Theta_{\bar{\epsilon}; t}\subseteq\Theta_{\epsilon}$, the proof completes.
\end{proof}

\section*{C: Experiments Details}

\subsection{Details of the synthetic datasets}

As mentioned in Section~\ref{sec: experiments}, we generate a synthetic dataset in which one feature is strongly correlated with the protected variable $D$ to induce a backdoor dependency on the outcome. We consider a binary protected variable, $D \in \{0,1\}$, which could indicate e.g., gender or race.
The synthetic dataset contains two features, a feature $X_1$ correlated with the protected variable and a feature $X_2$ uncorrelated with the protected variable.
For $D=0$, $X_1$ is uniformly distributed in $[0,10]$, while for $D=1$, $X_1 = 0$. Instead, $X_2$ is $5$ times a random variable from a  normal distribution $\mathcal{N}(0,1)$. Finally, the outcome $Y$ is binary, so $Y = \{0, 1\}$: $Y_i = 1$ when $Y_i > m_x + \eps_i$ and $Y_i = 0$ when $Y_i \le m_x + \eps_i$, where $m_x$ is the mean of $\{(X_1)_i + (X_2)_i\}_i$ and the noise $\eps_i$ comes from a normal distribution $\mathcal{N}(0,1)$

\subsection{More details on the speed comparison experiment.}
In the comparison with commercial solvers, all experiments run on an MacOS 13.0.1 machine, with 32G of RAM, and an Apple M1 Pro chip. 
We compare FairWASP (implemented in Python) with Gurobi 10.0 \cite{gurobi} and Mosek 10.1 \cite{mosek}. 
In problem \eqref{LP}, the pairwise distance $c(Z_i,Z_j)$ is defined as the Euclidean distance on the data after being normalized so that each feature has the same standard deviation.
The two commercial solvers use the default parameter setting (so that they can automatically choose the best method), except we restrict all methods to use only one thread for fairness of comparison. 
The parameter $\eps$ in constraints \eqref{type1} is set as $0.05$ for all different $n$.  FairWASP stops whenever the inner analytic center subproblem meets numerical issues or when the relative duality gap defined by $|obj_{p} - {obj}_d| / (1 + |obj_{p}| + |{obj}_d|)$ becomes smaller than or equal to $0.001$ or $10^{-3}$. Here $obj_p$ is $\langle C, P^k \rangle$ and $obj_d$ is $ - F(\lambda^k)$, for the dual solution $\lambda^k$ and its corresponding primal solution $P^k$.  The time limit is set as one hour for all methods and both the solvers exceed time limit $3,600$ seconds for the problem \eqref{MIP} for all tested different $n$. When $n$ is larger than $10^4$, the two solvers also exceed time limit even for \eqref{LP}. 

It should be noted that the solutions obtained by FairWASP are not as accurate as those obtained by solvers because they use different tolerances, but we will then show that the solutions obtained by FairWASP are already good enough for the goal of preprocessing. On the other hand, even if stopping the two solvers at a much lower accuracy level, they still require significantly longer time than FairWASP, due to the convergence behavior of the simplex and interior-point method and the much higher per iteration  complexity.

In addition to Figure \ref{fig:runtime} which compares the speed, we also define two error metrics for the solutions obtained by FairWASP. To be specific, the first error is relative objective gap, which is defined by 
$$
\frac{|\operatorname{obj}_{f} - \operatorname{obj}_{r}|}{|\operatorname{obj}_{f}|+|\operatorname{obj}_{r}| + 1} \ ,
$$
where $\operatorname{obj}_{f}$ is the objective value of the obtained FairWASP soltuion, and $\operatorname{obj}_{r}$ is the obtained objective value of Gurobi. The $1$ in the denominator is for preventing the occurrence of $0$ in the optimal objective. This objective error definition is commonly used as the termination criterion for commercial solvers, see the interior-point method for \citet{mosek}.

In addition to the relative objective gap, we also measure the feasibility error by the ``Fairness Violation''. To be specific, the fairness violation for weights $\theta \in \Delta_n$ is defined as:
$$
\begin{aligned}
\max_{d \in \calD, y \in \calY} \bigg\{0, \ & \frac{p_{T}(y)}{1 + \eps}  - p_{Z;\theta}(y|d) , \\
&   p_{Z;\theta}(y|d)  - (1 + \eps ) \cdot p_{T}(y)
\bigg\}.
\end{aligned}
$$
Note that the solution $\theta$ generated  by solving the \eqref{eq equivalent dual} will be guaranteed  to be in $\Delta_n$ so it is feasible if and only if the fairness violation is equal to zero. 

Figure \ref{fig:errors}  shows the quantities of these two kinds of errors of the FairWASP solutions for all different $n$ in $[100, 6400]$. It shows that both the relative objective gap and the fairness violation keep at a low level for all different $n$ and all different trials. This means although the FairWASP solutions are not as optimal as the solver solutions, they are still good enough for real applications. On the other hand, Figure \ref{fig:runtime} shows that the two commercial solvers are dramatically slower than FairWASP due to the matrix factorization used in the simplex or the interior-point methods.

\begin{figure}[htbp]
    \centering
    \includegraphics[width=1\columnwidth]{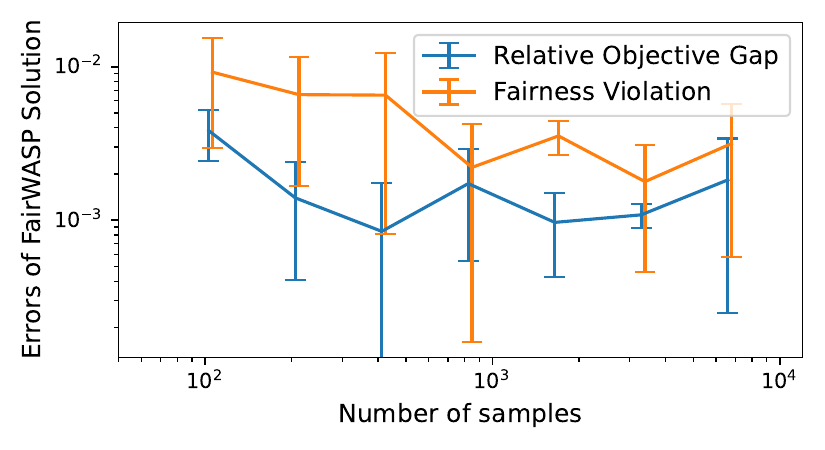}
    \caption{Relative objective gaps and fairness violations of the FairWASP solutions in the experiments in Figure \ref{fig:runtime}}
    \label{fig:errors}
\end{figure}

\subsection{Real Datasets Details}

We consider the following four real datasets widely used in the fairness literature \cite{fabris2022algorithmic}:

\begin{itemize}
    \item the \textit{Adult dataset} \cite{misc_adult_2}, which reports demographic data from the 1994 US demographic survey about $\sim 49,000$ individuals. We use all the available features for classification, including gender as the protected variable $D$ and whether the individual salary is more than $50,000$USD;
    \item the \textit{Drug dataset} \cite{fehrman2017five}, which contains drug usage history for $1885$ individuals. Features $X$ include the individual age, country of origin, education and scores on various psychological test. We use the individual gender as the protected variable $D$. The response $Y$ is based on whether the individual has reported to have never used the drug ``cannabis'' or not; 
    \item the \textit{Communities and Crime dataset} \cite{misc_communities_and_crime_183} was created towards the creation of a software tool for the US police department. The dataset contains socio-economic factors for $\sim 2,000$ communities in the US, along with the proportion of violent crimes in each community. As protected variable $D$, we include whether the percentage of the black population in the community is above the overall median. For the response $Y$, we use a binary indicator of whether the violent crimes percentage level is above the mean across all communities in the dataset;
    \item the \textit{German Credit dataset} \cite{misc_statlog_(german_credit_data)_144} reports a set of $1,000$ algorithmic credit decisions from a regional bank in Germany. We use all the available features, including gender as protected variable $D$ and whether the credit was approved as response $Y$.
\end{itemize}
Finally, we also evaluate the downstream fairness-utility tradeoff of the synthetic data generated in the previous subsection, with the number of samples equal to $n = 2,000$.

In addition, here below we describe the methods in details, along with the different hyper-parameters used for pre-processing:
\begin{enumerate}
    \item \textbf{Uniform}: Training on the empirical distribution as-is. There is no other parameters to adjust.
    \item \textbf{Disparate impact remover} (\textbf{DIR}, \citealt{feldman2015certifying}): A preprocessing technique that edits feature values to increase group fairness while preserving rank-ordering within groups.  This method contains one parameter that controls the repairing level, taking values between $0$ and $1$. In the experiments, we will show the results with this value being $0,  0.1,0.2,\dots,0.9$, $10$ different values. 
    \item \textbf{Learning fair representations} (\textbf{LFR}, \citealt{zemel2013learning}): A preprocessing technique that finds a latent representation which encodes the data well but obfuscates information about protected attributes.  This method contains different parameters $Ax$, $Ay$, $Az$ controlling the weight of the input recontruction quality term, the weight of the fairness constraint term and the output prediction error. \cite{zemel2013learning} reports the best result from setups $Ax = 0.01$ and $Ay,Az\in \{0.1, 0.5, 1, 5, 10\}$. In this paper, for simplicity of the results, we report all the results of using parameter setups from $Ax = 0.01$ and $Ay,Az\in \{0.1, 1, 10\}$.
    \item \textbf{Reweighing} \cite{kamiran2012data}: A preprocessing technique that weights the examples in each (group, label) combination differently to ensure fairness before classification. There is no other parameters to adjust.
    \item \textbf{Optimized preprocessing} \cite{calmon2017optimized}:  A preprocessing technique that learns a probabilistic transformation that edits the features and labels in the data with group fairness, individual distortion, and data fidelity constraints and objectives.  This method is only applicable for the dataset whose overall number of different individual samples is relatively small so usually all data should be first converted into a binary smaller dataset, then apply pre-processing on the binarized dataset, and finally train on the downstream model on the pre-processed data. Similarly, the test set should also be binarized for applying the trained model on it. There are multiple parameters to adjust in \cite{calmon2017optimized} so we only compare with it on the dataset Adult because the authors of \cite{calmon2017optimized} provide the binarized Adult dataset and  the setup of all the parameters. 
    \item \textbf{FairWASP}: Our method with fairness constraints \eqref{type1}, with the parwise distance $c(Z_i,Z_j)$ defined as the Euclidean distance on the standardized features. We report the results of setting $\eps$ of $\eqref{type1}$ as all values  in $\{0.001,0.01,0.1,0.2,0.3\}$.
    \item \textbf{FairWASP-GW}: Our method with group-wise fairness constraints \eqref{type2}, with the parwise distance $c(Z_i,Z_j)$ defined as the Euclidean distance on the standardized features. We report the results of setting $\eps$ of $\eqref{type2}$ as all values  in $\{0.001,0.01,0.1,0.2,0.3\}$.
\end{enumerate}

All computations are run on an Ubuntu machine with 32GB of RAM and 2.50GHz Intel(R) Xeon(R) Platinum 8259CL CPU. For all datasets, we randomly split $75\%$ of the data into training/test set, and change the split during each separate run; the training data are further separated into training and validation with $90/10$ to compute early stopping criteria during training. The downstream classifier used is a one-layer deep multi-layer perceptron with $20$ hidden layers, ReLu activation function in the hidden layer and softmax activation function in the final layer. The learning rate is set to $10^{-3}$, with a batch size of $32$, and a maximum number of epochs set to $500$ with early stopping evaluated on the separate validation set with a patience of $10$ epochs.
Figures~\ref{fig:resultsreal} and \ref{fig:allresults uniform distance} report the fairness-utility tradeoff indicated by demographic parity and AUC of the classifier respectively, with each scatter indicating one method-parameter pair. Each point is the average results over 10 separate runs, with the bars indicating the standard deviation. Note that in some runs the MLP classifier collapses into the trivial classifier, always returning the same outcome regardless of the input. Such trivial classifier is fair by definition, but achieves a $0.5$ AUC and is responsible for the large error bars, especially in the German Credit dataset. Overall, FairWASP and FairWASP-GW achieve competitive performance with existing pre-processing methods, often achieving similar utility as using the original empirical distribution but with much better demographic parity.

\begin{figure*}[!ht]
    \centering
    \includegraphics[width=0.49\textwidth]{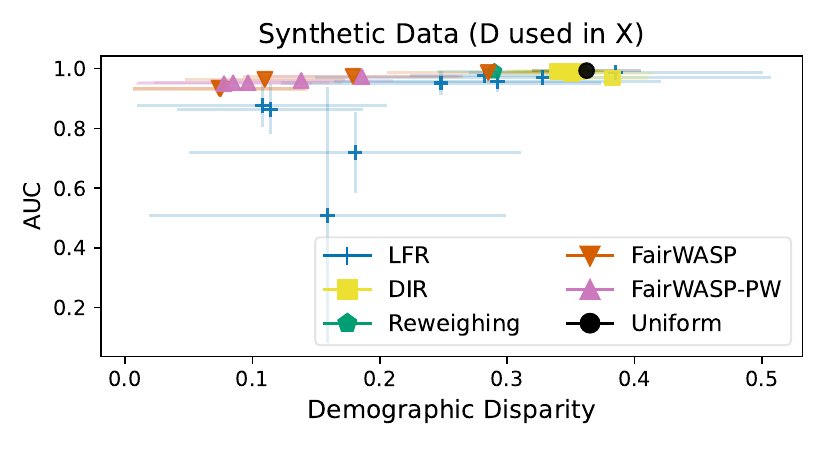} \hfill
    \includegraphics[width=0.49\textwidth]{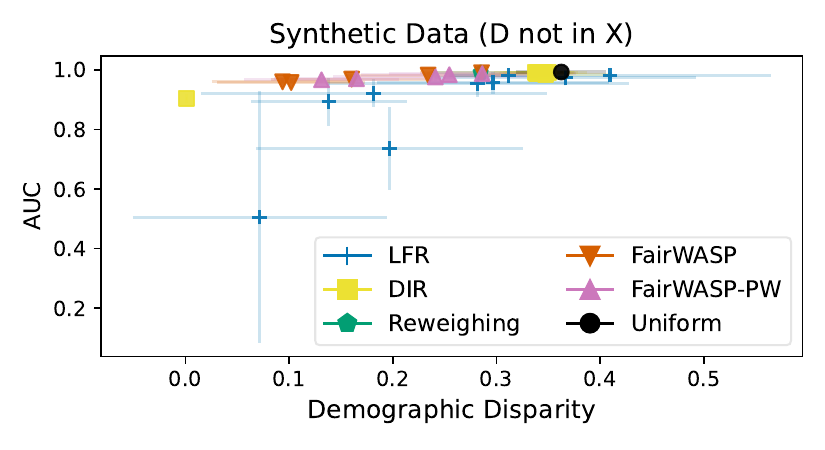} \\
    \includegraphics[width=0.49\textwidth]{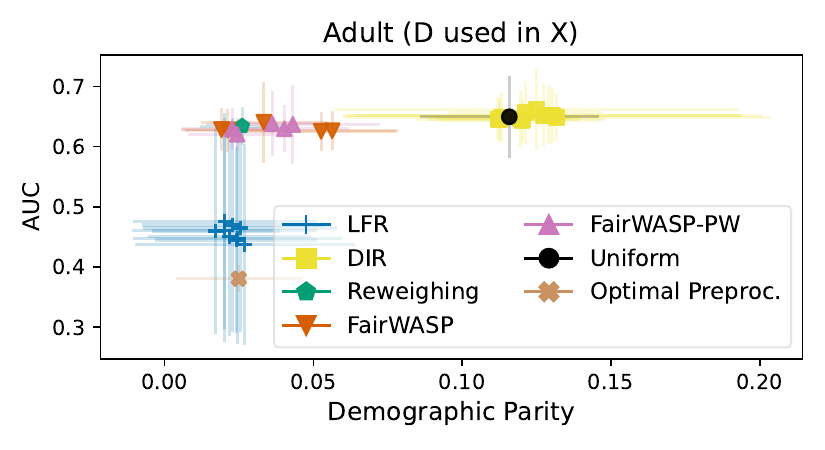} \hfill
    \includegraphics[width=0.49\textwidth]{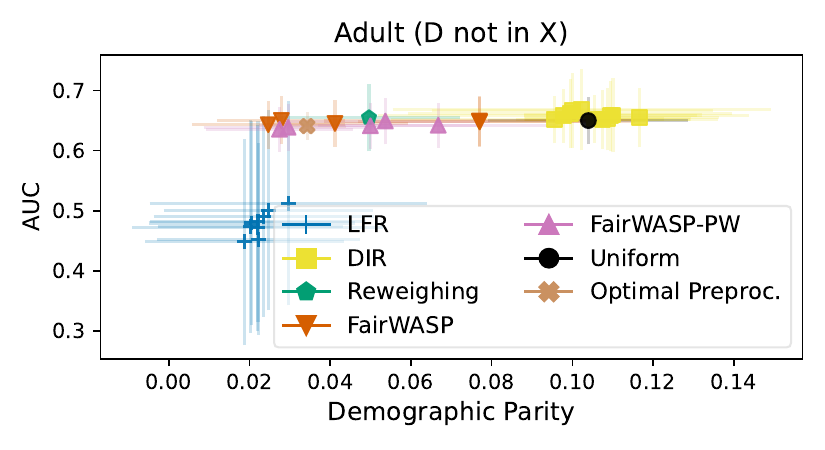} \\
    \includegraphics[width=0.49\textwidth]{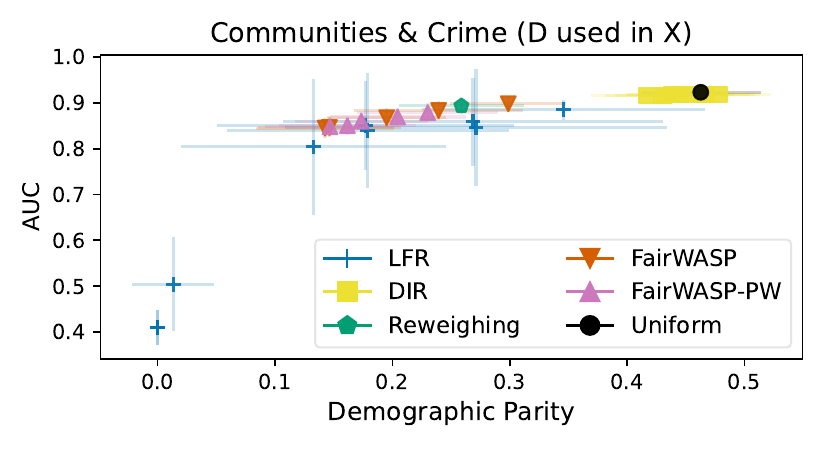} \hfill
    \includegraphics[width=0.49\textwidth]{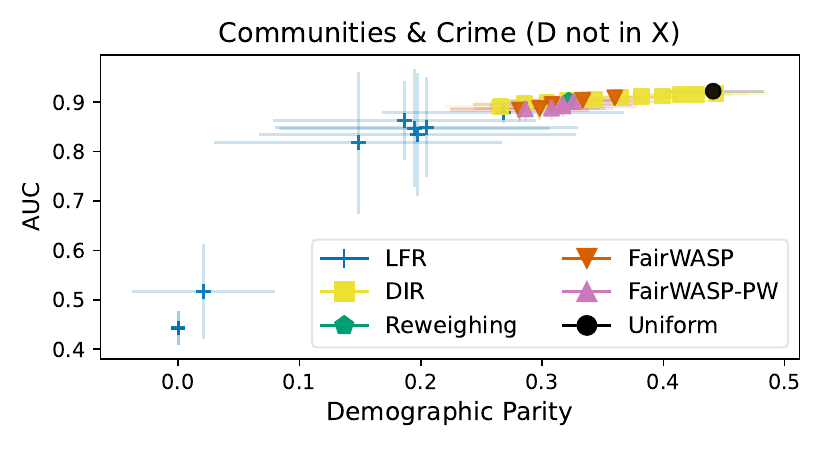} \\
    \includegraphics[width=0.49\textwidth]{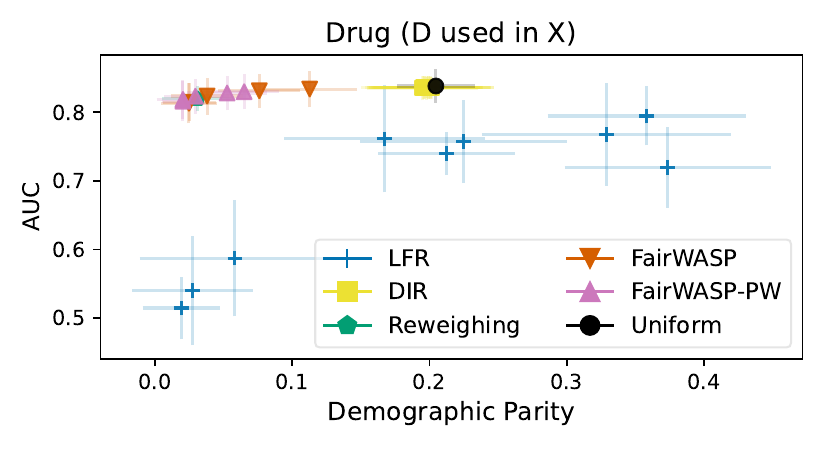} \hfill
    \includegraphics[width=0.49\textwidth]{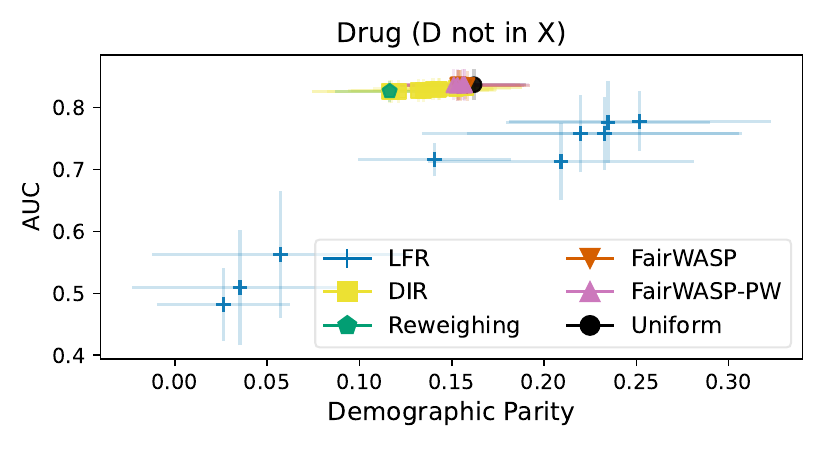} \\
    \includegraphics[width=0.49\textwidth]{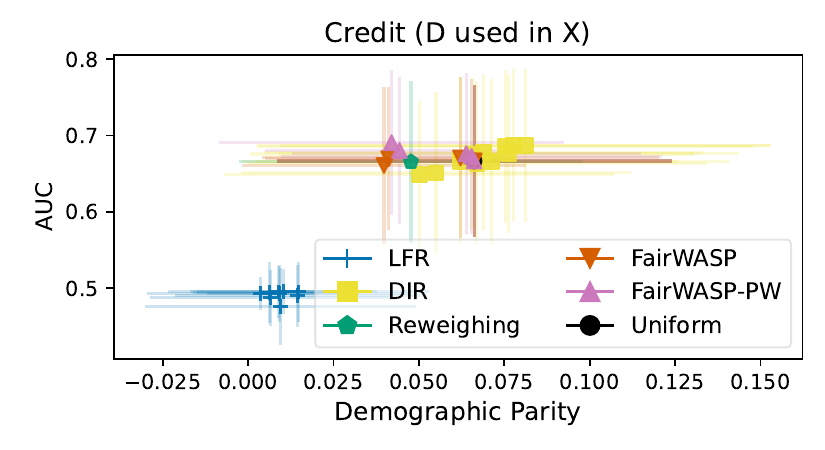} \hfill
    \includegraphics[width=0.49\textwidth]{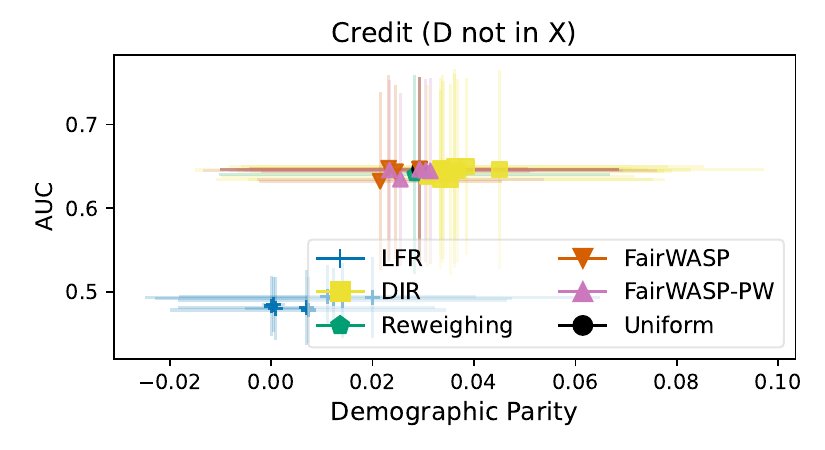} \\
    \caption{Downstream fairness-utility tradeoff, indicated by the demographic parity and downstream classifier area under the curve (AUC). Points and error bars correspond to averages plus/minus one standard deviation, computed over 10 different train/test split. Each point is a method-parameter combination. See text for more details.}
    \label{fig:allresults uniform distance}
\end{figure*}


\begin{thebibliography}{52}
\providecommand{\natexlab}[1]{#1}

\bibitem[{Applegate et~al.(2021)Applegate, D{\'\i}az, Hinder, Lu, Lubin,
  O'Donoghue, and Schudy}]{applegate2021practical}
Applegate, D.; D{\'\i}az, M.; Hinder, O.; Lu, H.; Lubin, M.; O'Donoghue, B.;
  and Schudy, W. 2021.
\newblock Practical large-scale linear programming using primal-dual hybrid
  gradient.
\newblock \emph{Advances in Neural Information Processing Systems}, 34:
  20243--20257.

\bibitem[{Applegate et~al.(2022)Applegate, Hinder, Lu, and
  Lubin}]{applegate2022faster}
Applegate, D.; Hinder, O.; Lu, H.; and Lubin, M. 2022.
\newblock Faster first-order primal-dual methods for linear programming using
  restarts and sharpness.
\newblock \emph{Mathematical Programming}, 1--52.

\bibitem[{Atkinson and Vaidya(1995)}]{atkinson1995cutting}
Atkinson, D.~S.; and Vaidya, P.~M. 1995.
\newblock A cutting plane algorithm for convex programming that uses analytic
  centers.
\newblock \emph{Mathematical Programming}, 69(1-3): 1--43.

\bibitem[{Bachem, Lucic, and Krause(2017)}]{bachem2017practical}
Bachem, O.; Lucic, M.; and Krause, A. 2017.
\newblock Practical coreset constructions for machine learning.
\newblock \emph{arXiv preprint arXiv:1703.06476}.

\bibitem[{Becker and Kohavi(1996)}]{misc_adult_2}
Becker, B.; and Kohavi, R. 1996.
\newblock {Adult}.
\newblock UCI Machine Learning Repository.
\newblock {DOI}: https://doi.org/10.24432/C5XW20.

\bibitem[{Bertsimas and Tsitsiklis(1997)}]{bertsimas1997introduction}
Bertsimas, D.; and Tsitsiklis, J.~N. 1997.
\newblock \emph{Introduction to linear optimization}, volume~6.
\newblock Athena Scientific Belmont, MA.

\bibitem[{Boyd, Vandenberghe, and Skaf(2018)}]{Boydanalyticcenter}
Boyd, S.; Vandenberghe, L.; and Skaf, J. 2018.
\newblock Analytic Center Cutting-Plane Method.
\newblock Lecture note, Dept.\ of Electrical Engineering, Stanford Univ.
\newblock Available on
  \url{https://web.stanford.edu/class/ee364b/lectures/accpm_notes.pdf}.
  Accessed: 2023-08-03.

\bibitem[{Calders, Kamiran, and
  Pechenizkiy(2009)}]{calders2009BuildingClassifiersIndependency}
Calders, T.; Kamiran, F.; and Pechenizkiy, M. 2009.
\newblock Building classifiers with independency constraints.
\newblock In \emph{Proceedings of the 2009 IEEE International Conference on
  Data Mining Workshops}, 13--18. IEEE.

\bibitem[{Calmon et~al.(2017)Calmon, Wei, Vinzamuri, Natesan~Ramamurthy, and
  Varshney}]{calmon2017optimized}
Calmon, F.; Wei, D.; Vinzamuri, B.; Natesan~Ramamurthy, K.; and Varshney, K.~R.
  2017.
\newblock Optimized pre-processing for discrimination prevention.
\newblock \emph{Proceedings of the 30th Advances in Neural Information
  Processing Systems}.

\bibitem[{Chai and Wang(2022)}]{chai2022FairnessAdaptiveWeights}
Chai, J.; and Wang, X. 2022.
\newblock Fairness with adaptive weights.
\newblock In \emph{Proceedings of the 39th International Conference on Machine
  Learning}, 2853--2866. PMLR.

\bibitem[{Chakraborty, Majumder, and
  Menzies(2021)}]{chakraborty2021BiasMachineLearning}
Chakraborty, J.; Majumder, S.; and Menzies, T. 2021.
\newblock Bias in machine learning software: {W}hy? {H}ow? {W}hat to do?
\newblock In \emph{Proceedings of the 29th ACM Joint Meeting on European
  Software Engineering Conference and Symposium on the Foundations of Software
  Engineering}, 429--440.

\bibitem[{Chang(2009)}]{chang2009n}
Chang, Y.-C. 2009.
\newblock N-dimension golden section search: {I}ts variants and limitations.
\newblock In \emph{Proceedings of the 2nd International Conference on
  Biomedical Engineering and Informatics}, 1--6. IEEE.

\bibitem[{Chzhen et~al.(2020)Chzhen, Denis, Hebiri, Oneto, and
  Pontil}]{chzhen2020FairRegressionWasserstein}
Chzhen, E.; Denis, C.; Hebiri, M.; Oneto, L.; and Pontil, M. 2020.
\newblock Fair regression with wasserstein barycenters.
\newblock \emph{Proceedings of the 33rd Advances in Neural Information
  Processing Systems}, 33: 7321--7331.

\bibitem[{Chzhen and Schreuder(2022)}]{chzhen2022minimax}
Chzhen, E.; and Schreuder, N. 2022.
\newblock A minimax framework for quantifying risk-fairness trade-off in
  regression.
\newblock \emph{The Annals of Statistics}, 50(4): 2416--2442.

\bibitem[{Claici, Genevay, and Solomon(2018)}]{claici2018wasserstein}
Claici, S.; Genevay, A.; and Solomon, J. 2018.
\newblock Wasserstein measure coresets.
\newblock \emph{arXiv preprint arXiv:1805.07412}.

\bibitem[{Du and Pardalos(1995)}]{du1995minimax}
Du, D.-Z.; and Pardalos, P.~M. 1995.
\newblock \emph{Minimax and applications}, volume~4.
\newblock Springer Science \& Business Media.

\bibitem[{Dwork et~al.(2012)Dwork, Hardt, Pitassi, Reingold, and
  Zemel}]{dwork2012FairnessAwareness}
Dwork, C.; Hardt, M.; Pitassi, T.; Reingold, O.; and Zemel, R. 2012.
\newblock Fairness through awareness.
\newblock In \emph{Proceedings of the 3rd Innovations in Theoretical Computer
  Science Conference}, 214--226.

\bibitem[{Fabris et~al.(2022)Fabris, Messina, Silvello, and
  Susto}]{fabris2022algorithmic}
Fabris, A.; Messina, S.; Silvello, G.; and Susto, G.~A. 2022.
\newblock Algorithmic fairness datasets: {T}he story so far.
\newblock \emph{Data Mining and Knowledge Discovery}, 36(6): 2074--2152.

\bibitem[{Fehrman et~al.(2017)Fehrman, Muhammad, Mirkes, Egan, and
  Gorban}]{fehrman2017five}
Fehrman, E.; Muhammad, A.~K.; Mirkes, E.~M.; Egan, V.; and Gorban, A.~N. 2017.
\newblock The five factor model of personality and evaluation of drug
  consumption risk.
\newblock In \emph{Data Science: Innovative Developments in Data Analysis and
  Clustering}, 231--242. Springer.

\bibitem[{Feldman et~al.(2015)Feldman, Friedler, Moeller, Scheidegger, and
  Venkatasubramanian}]{feldman2015certifying}
Feldman, M.; Friedler, S.~A.; Moeller, J.; Scheidegger, C.; and
  Venkatasubramanian, S. 2015.
\newblock Certifying and removing disparate impact.
\newblock In \emph{Proceedings of the 21th ACM SIGKDD International Conference
  on Knowledge Discovery and Data Mining}, 259--268.

\bibitem[{Gao and Han(2012)}]{gao2012implementing}
Gao, F.; and Han, L. 2012.
\newblock Implementing the {N}elder-{M}ead simplex algorithm with adaptive
  parameters.
\newblock \emph{Computational Optimization and Applications}, 51(1): 259--277.

\bibitem[{Ge et~al.(2022)Ge, Zhao, Yu, Paul, Hu, Hsieh, and
  Zhang}]{ge2022toward}
Ge, Y.; Zhao, X.; Yu, L.; Paul, S.; Hu, D.; Hsieh, C.-C.; and Zhang, Y. 2022.
\newblock Toward Pareto efficient fairness-utility trade-off in recommendation
  through reinforcement learning.
\newblock In \emph{Proceedings of the 15th ACM International Conference on Web
  Search and Data Mining}, 316--324.

\bibitem[{{Gurobi Optimization, LLC}(2023)}]{gurobi}
{Gurobi Optimization, LLC}. 2023.
\newblock {Gurobi Optimizer Reference Manual}.
\newblock Available on \url{https://www.gurobi.com} Accessed: 2023-08-03.

\bibitem[{Hofmann(1994)}]{misc_statlog_(german_credit_data)_144}
Hofmann, H. 1994.
\newblock {Statlog (German Credit Data)}.
\newblock UCI Machine Learning Repository.
\newblock {DOI}: https://doi.org/10.24432/C5NC77.

\bibitem[{Hort et~al.(2022)Hort, Chen, Zhang, Sarro, and
  Harman}]{hort2022BiasMitigationMachinea}
Hort, M.; Chen, Z.; Zhang, J.~M.; Sarro, F.; and Harman, M. 2022.
\newblock Bias mitigation for machine learning classifiers: {A} comprehensive
  survey.
\newblock \emph{arXiv preprint arXiv:2207.07068}.

\bibitem[{Jiang et~al.(2020)Jiang, Lee, Song, and Wong}]{jiang2020improved}
Jiang, H.; Lee, Y.~T.; Song, Z.; and Wong, S. C.-w. 2020.
\newblock An improved cutting plane method for convex optimization,
  convex-concave games, and its applications.
\newblock In \emph{Proceedings of the 52nd Annual ACM SIGACT Symposium on
  Theory of Computing}, 944--953.

\bibitem[{Jiang and Nachum(2020)}]{jiang2020IdentifyingCorrectingLabel}
Jiang, H.; and Nachum, O. 2020.
\newblock Identifying and correcting label bias in machine learning.
\newblock In \emph{Proceedings of the 23rd International Conference on
  Artificial Intelligence and Statistics}, 702--712. PMLR.

\bibitem[{Kamiran and
  Calders(2010)}]{kamiran2010ClassificationNoDiscrimination}
Kamiran, F.; and Calders, T. 2010.
\newblock Classification with no discrimination by preferential sampling.
\newblock In \emph{Proceedings of the 19th Machine Learning Conference of
  Belgium and The Netherlands}, volume~1. Citeseer.

\bibitem[{Kamiran and Calders(2012)}]{kamiran2012data}
Kamiran, F.; and Calders, T. 2012.
\newblock Data preprocessing techniques for classification without
  discrimination.
\newblock \emph{Knowledge and Information Systems}, 33(1): 1--33.

\bibitem[{Kantorovitch(1958)}]{kantorovitch1958translocation}
Kantorovitch, L. 1958.
\newblock On the translocation of masses.
\newblock \emph{Management Science}, 5(1): 1--4.

\bibitem[{Khachiyan(1980)}]{Khachiyan1980polynomial}
Khachiyan, L.~G. 1980.
\newblock Polynomial algorithms in linear programming.
\newblock \emph{USSR Computational Mathematics and Mathematical Physics},
  20(1): 53--72.

\bibitem[{Klee and Minty(1972)}]{klee1972good}
Klee, V.; and Minty, G.~J. 1972.
\newblock How good is the simplex algorithm.
\newblock \emph{Inequalities}, 3(3): 159--175.

\bibitem[{Lee and Sidford(2014)}]{lee2014path}
Lee, Y.~T.; and Sidford, A. 2014.
\newblock Path finding methods for linear programming: Solving linear programs
  in $\tilde{O}(\sqrt{T})$ iterations and faster algorithms for maximum flow.
\newblock In \emph{2014 IEEE 55th Annual Symposium on Foundations of Computer
  Science}, 424--433. IEEE.

\bibitem[{Li and Liu(2022)}]{li2022AchievingFairnessNo}
Li, P.; and Liu, H. 2022.
\newblock Achieving fairness at no utility cost via data reweighing with
  influence.
\newblock In \emph{Proceedings of the 39th International Conference on Machine
  Learning}, 12917--12930. PMLR.

\bibitem[{Memoli, Smith, and Wan(2019)}]{memoli19awasstransf}
Memoli, F.; Smith, Z.; and Wan, Z. 2019.
\newblock The {W}asserstein Transform.
\newblock In Chaudhuri, K.; and Salakhutdinov, R., eds., \emph{Proceedings of
  the 36th International Conference on Machine Learning}, volume~97 of
  \emph{Proceedings of Machine Learning Research}, 4496--4504. PMLR.

\bibitem[{{MOSEK ApS}(2023)}]{mosek}
{MOSEK ApS}. 2023.
\newblock \emph{MOSEK Optimization Suite}.
\newblock Available on \url{https://docs.mosek.com/10.0/intro.pdf} Accessed:
  2023-08-03.

\bibitem[{Nesterov(2018)}]{nesterov2018lectures}
Nesterov, Y. 2018.
\newblock \emph{Lectures on convex optimization}, volume 137.
\newblock Springer.

\bibitem[{Peng and Vempala(2021)}]{peng2021solving}
Peng, R.; and Vempala, S. 2021.
\newblock Solving sparse linear systems faster than matrix multiplication.
\newblock In \emph{Proceedings of the 2021 ACM-SIAM Symposium on Discrete
  Algorithms}, 504--521. SIAM.

\bibitem[{Peyr{\'e}, Cuturi et~al.(2019)}]{peyre2019computational}
Peyr{\'e}, G.; Cuturi, M.; et~al. 2019.
\newblock Computational optimal transport: {W}ith applications to data science.
\newblock \emph{Foundations and Trends{\textregistered} in Machine Learning},
  11(5-6): 355--607.

\bibitem[{Redmond(2009)}]{misc_communities_and_crime_183}
Redmond, M. 2009.
\newblock {Communities and Crime}.
\newblock UCI Machine Learning Repository.
\newblock {DOI}: https://doi.org/10.24432/C53W3X.

\bibitem[{Renegar(2001)}]{renegar2001mathematical}
Renegar, J. 2001.
\newblock \emph{A mathematical view of interior-point methods in convex
  optimization}.
\newblock SIAM.

\bibitem[{Salazar et~al.(2021)Salazar, Santos, Ara{\'u}jo, and
  Abreu}]{salazar2021FAWOSFairnessAwareOversampling}
Salazar, T.; Santos, M.~S.; Ara{\'u}jo, H.; and Abreu, P.~H. 2021.
\newblock {FAWOS}: {F}airness-aware oversampling algorithm based on
  distributions of sensitive attributes.
\newblock \emph{IEEE Access}, 9: 81370--81379.

\bibitem[{Salimi et~al.(2019)Salimi, Rodriguez, Howe, and
  Suciu}]{salimi2019InterventionalFairnessCausal}
Salimi, B.; Rodriguez, L.; Howe, B.; and Suciu, D. 2019.
\newblock Interventional fairness: {C}ausal database repair for algorithmic
  fairness.
\newblock In \emph{Proceedings of the 2019 International Conference on
  Management of Data}, 793--810.

\bibitem[{Santambrogio(2015)}]{santambrogio2015optimal}
Santambrogio, F. 2015.
\newblock Optimal transport for applied mathematicians.
\newblock \emph{Birk{\"a}user, NY}, 55(58-63): 94.

\bibitem[{Sloane, Moss, and Chowdhury(2022)}]{sloane2022SiliconValleyLove}
Sloane, M.; Moss, E.; and Chowdhury, R. 2022.
\newblock A Silicon Valley love triangle: {H}iring algorithms, pseudo-science,
  and the quest for auditability.
\newblock \emph{Patterns}, 3(2).

\bibitem[{Villani et~al.(2009)}]{villani2009optimal}
Villani, C.; et~al. 2009.
\newblock \emph{Optimal transport: {O}ld and new}, volume 338.
\newblock Springer.

\bibitem[{Wang et~al.(2022)Wang, Cheng, Basu, Gupta, Selvaraj, and
  Mazumder}]{wang2022light}
Wang, H.; Cheng, M.; Basu, K.; Gupta, A.; Selvaraj, K.; and Mazumder, R. 2022.
\newblock A Light-speed Linear Program Solver for Personalized Recommendation
  with Diversity Constraints.
\newblock \emph{arXiv preprint arXiv:2211.12409}.

\bibitem[{Xu et~al.(2018)Xu, Yuan, Zhang, and
  Wu}]{xu2018FairGANFairnessawareGenerativea}
Xu, D.; Yuan, S.; Zhang, L.; and Wu, X. 2018.
\newblock {FairGAN}: {F}airness-aware generative adversarial networks.
\newblock In \emph{Proceedings of the 2018 IEEE International Conference on Big
  Data}, 570--575. IEEE.

\bibitem[{Yan, Kao, and Ferrara(2020)}]{yan2020FairClassBalancing}
Yan, S.; Kao, H.-t.; and Ferrara, E. 2020.
\newblock Fair class balancing: {E}nhancing model fairness without observing
  sensitive attributes.
\newblock In \emph{Proceedings of the 29th ACM International Conference on
  Information \& Knowledge Management}, 1715--1724.

\bibitem[{Zemel et~al.(2013)Zemel, Wu, Swersky, Pitassi, and
  Dwork}]{zemel2013learning}
Zemel, R.; Wu, Y.; Swersky, K.; Pitassi, T.; and Dwork, C. 2013.
\newblock Learning fair representations.
\newblock In \emph{Proceedings of the 30th International Conference on Machine
  Learning}, 325--333. PMLR.

\bibitem[{Zhang et~al.(2022)Zhang, Xing, Zou, and
  Wu}]{zhang2022ShiftingMachineLearning}
Zhang, A.; Xing, L.; Zou, J.; and Wu, J.~C. 2022.
\newblock Shifting Machine Learning for Healthcare from Development to
  Deployment and from Models to Data.
\newblock \emph{Nature Biomedical Engineering}, 6(12): 1330--1345.

\bibitem[{{\v{Z}}liobaite, Kamiran, and
  Calders(2011)}]{zliobaite2011HandlingConditionalDiscrimination}
{\v{Z}}liobaite, I.; Kamiran, F.; and Calders, T. 2011.
\newblock Handling conditional discrimination.
\newblock In \emph{Proceedings of the 11th IEEE International Conference on
  Data Mining}, 992--1001. IEEE.

\end{thebibliography}
\end{document}